\documentclass[11pt,a4paper,titlepage]{memoir}

\usepackage[OT1]{fontenc}

\usepackage[english]{babel}

\usepackage[utf8]{inputenc}

\usepackage[sc]{mathpazo}

\usepackage{amsmath,amssymb,amsfonts,mathrsfs}

\usepackage[amsmath,thmmarks]{ntheorem}

\usepackage{graphicx}

\usepackage{soul}

\usepackage{pdfpages}

\usepackage{varioref}

\usepackage{datetime}

\usepackage{mathtools}

\usepackage[h]{esvect}

\usepackage{array}

\usepackage{listings}
\usepackage{booktabs}

\usepackage[ruled,vlined]{algorithm2e}

\usepackage{ETHlogo}

\nonzeroparskip
\defaultlists

\makeatletter

\if@twoside
  \copypagestyle{chapter}{Ruled}
\else
  \copypagestyle{chapter}{ruled}
\fi
\makeoddhead{chapter}{}{}{}
\makeevenhead{chapter}{}{}{}
\makeheadrule{chapter}{\textwidth}{0pt}
\copypagestyle{abstract}{empty}

\makechapterstyle{bianchimod}{%
  \chapterstyle{default}
  \renewcommand*{\chapnamefont}{\normalfont\Large\sffamily}
  
  \renewcommand*{\printchaptername}{%
    \chapnamefont\centering\@chapapp}

  }

\chapterstyle{bianchimod}

\setsecheadstyle{\Large\bfseries\sffamily}
\setsubsecheadstyle{\large\bfseries\sffamily}
\setsubsubsecheadstyle{\bfseries\sffamily}
\setparaheadstyle{\normalsize\bfseries\sffamily}
\setsubparaheadstyle{\normalsize\itshape\sffamily}
\setsubparaindent{0pt}

\captionnamefont{\sffamily\bfseries\footnotesize}
\captiontitlefont{\sffamily\footnotesize}
\setsecnumdepth{subsection}
\settocdepth{subsection}

\pretitle{\vspace{0pt plus 0.7fill}\begin{center}\HUGE\sffamily\bfseries}
\posttitle{\end{center}\par}
\preauthor{\par\begin{center}\let\and\\\Large\sffamily}
\postauthor{\end{center}}
\predate{\par\begin{center}\Large\sffamily}
\postdate{\end{center}}

\def\@advisors{}
\newcommand{\advisors}[1]{\def\@advisors{#1}}
\def\@department{}
\newcommand{\department}[1]{\def\@department{#1}}
\def\@thesistype{}
\newcommand{\thesistype}[1]{\def\@thesistype{#1}}

\renewcommand{\maketitlehookb}{\vspace{1in}%
  \par\begin{center}\Large\sffamily\@thesistype\end{center}}

\renewcommand{\maketitlehookd}{%
  \vfill\par
  \begin{flushright}
    \sffamily
    \@advisors\par
    \@department
  \end{flushright}
}

\checkandfixthelayout

\makeatother

\theoremstyle{plain}
\numberwithin{equation}{chapter}

\newtheorem{theorem}{Theorem}[chapter]

\newtheorem{remark}[theorem]{Remark}

\newtheorem{definition}[theorem]{Definition}
\newtheorem{lemma}[theorem]{Lemma}
\newtheorem{proposition}[theorem]{Proposition}

\theoremstyle{nonumberplain}
\theorembodyfont{\normalfont}
\theoremsymbol{\ensuremath{\square}}
\newtheorem{proof}{Proof}

\newcommand{\N}{\mathbb{N}}

\newcommand{\R}{\mathbb{R}}

\renewcommand{\epsilon}{\ensuremath\varepsilon}

\DeclareMathOperator*{\argmax}{arg\,max}
\DeclareMathOperator*{\argmin}{arg\,min}

\usepackage[linkcolor=black,colorlinks=true,citecolor=black,filecolor=black]{hyperref}

\title{Competition analysis on the over-the-counter credit default swap market}
\author{Louis Abraham}
\date{July 2020}

\begin{document}

\frontmatter

\begin{titlingpage}
  \calccentering{\unitlength}
  \begin{adjustwidth*}{\unitlength-24pt}{-\unitlength-24pt}
    \maketitle
  \end{adjustwidth*}
\end{titlingpage}

\begin{abstract}
  We study two questions related to competition on the OTC CDS market using data collected as part of the EMIR regulation.
  
  First, we wanted to study the competition between central counterparties through collateral requirements. We present models that successfully estimate the initial margin requirements. However, our estimations are not precise enough to use them as input to a predictive model for CCP choice by counterparties in the OTC market.
  
  Second, we model counterpart choice on the interdealer market using a novel semi-supervised predictive task. We present our methodology as part of the literature on model interpretability before arguing for the use of conditional entropy as the metric of interest to derive knowledge from data through a model-agnostic approach. In particular, we justify the use of deep neural networks to measure conditional entropy on real-world datasets. We created the \textit{Razor entropy} using the framework of algorithmic information theory and derived an explicit formula that is identical to our semi-supervised training objective. Finally, we borrowed concepts from game theory to define \textit{top-k Shapley values}. This novel method of payoff distribution satisfies most of the properties of Shapley values, and is of particular interest when the value function is monotone submodular. Unlike (classical) Shapley values, \textit{top-k Shapley values} can be computed in quadratic time of the number of features instead of exponential. We implemented our methodology and report the results on our particular task of counterpart choice.
  
  Finally, we present an improvement to the \textit{node2vec} algorithm that we intended to use to study the intermediation in the OTC market. We show that the neighbor sampling used in the generation of biased walks can be performed in logarithmic time with a quasilinear time (and space) pre-computation, unlike the current implementations that do not scale well.
\end{abstract}

\cleartorecto
\tableofcontents
\mainmatter

\chapter{Introduction}

The goal of this work is to apply machine learning techniques to the field of economics by studying financial market data from the point of view of machine learning interpretability and produce novel economical insight.

In this introduction, we start by explaining the context, the different entities and concepts that are to be analysed and the data sources. We then formulate the questions we tackled during this research project and explicit our other contributions.

\section{Context}

The data used in this work is essentially information about transactions from the over-the-counter (OTC) derivative market.
Our analysis was concentrated for simplicity on a special class of products, Credit Default Swaps ; we will briefly explain their mechanics, market dynamics and pricing models.
A fundamental concept when dealing with counterparty risk on OTC markets is central counterparty clearing that we define. Finally, we present the European market infrastructure regulation (EMIR) according to which our data was collected by European regulators.

\subsection{OTC derivative markets}

\paragraph{Over-the-counter trading}

Over-the-counter, as opposed to exchange trading, refers to the trading of goods (commodities, financial instrument and derivatives) done bilaterally between two entities, without the centrality of an exchange. Unlike exchange trading, there is no common place where traders can display or fetch buying/selling (bid/ask) prices.

\paragraph{Derivatives}

On the other side, the OTC market allows for a wider choice of products, because they don't have to be standardized. In particular, the OTC derivative market is important as it encompasses a lot of exotic products that cannot be traded on exchanges. These derivatives are much more flexible in the way they hedge the risk exposure. Most derivatives are defined by agreements  made by the International Swaps and Derivatives Association.

An over-the-counter trade is a bilateral contract between two parties (usually, financial companies) that defines some future transfers between them in the future. Exchange-traded derivatives are bilateral contracts too, but on OTC markets, the negotiation was made between the two parties without other participants being able to observe. For example, a \textit{forward} contract involves two parties, one in a long position and the other in a short position. They agree that the long party will buy a given asset from the short party at a given time in the future for a given price.

The market for OTC derivative is one of the largest in the world, and as grown exponentially since the 90~\cite{remolona1992recent}.

\subsection{Credit Default Swaps}

Credit default swaps (CDS) are an example of derivative contracts. They involve two parties, a buyer and a seller. The CDS is linked to a reference entity (e.g. a company or a government). The buyer of the CDS is bound to make regular payments to the seller until the contract expiration date. These payments are proportional to a theoretical value of the underlying bond, called \textit{notional}.
In return, in the event of credit events on the reference entity, auctions on the settlement price of CDS take place between the major financial institutions. Unlike coupons, the settlement price of CDS has a magnitude similar to the notional. The \textit{notional} is also the maximum amount payable for a credit event.

Under the hood, CDS are intended to be bought by holders of bonds emitted by the reference entity, but it is in no way an obligation. A bond is a debt instrument, a form of loan. The bond issuer (the reference entity) borrows money to the people who hold the bond. The original amount of money borrowed is called the \textit{nominal}, principal, par or face amount. The bond issuer will repay the nominal bond value at the \textit{maturity} date, after which the contract is terminated. Furthermore, coupons are regular payments of the issuer to the bond holder that repay the loan interest. In case the bond issuer goes bankrupt, it cannot pay the coupons nor the nominal.

Hence, a credit default swap can be viewed as a way to hedge (eliminate the risk) a bond. Indeed, in exchange for regular payments (that will decrease the interest rate of the bond), a bond holder is guaranteed not to lose the nominal amount. Note that anyone can buy a CDS, not only bond holders.

An important quantity is the CDS \textit{spread}, not to be confused with the bid/ask spread of a product, which is simply the difference between the offer and demand prices on a market. Instead, the spread of a CDS contract is the fraction of the notional that is to be paid annually (through coupons). Intuitively, the more risky a reference entity is, the higher the spread of the CDS linked to it will be. Spreads are expressed in basis point, where 1 basis point is $0.01\%$. For example, the buyer of a CDS with a notional of \$10 million with a 150 basis point spread is expected to make annual payments of \$150k (usually split in quarterly coupons).

More complex phenomenon arise when one considers the secondary market, aka when people want to transfer the ownership of a CDS. Without the ISDA standardization, it would be difficult to trade very different contracts, as they could include a lot of different clauses, e.g. the condition to declare default of the reference entity, the seniority (priority) of the underlying bond, recovery rates, \dots

Thankfully, the standard ISDA allows for a clear pricing model. Note that the expected value of a CDS is always zero for both sides. Theoretically, the spread should be adjusted to match the expected bond risk. However, to allow standardization, most CDS have a 1\% or 5\% coupon, which implies that \textit{upfront payments} are to be made to compensate the spread difference. Note that this upfront payment can be positive or negative, depending on the difference between the clean spread and the standardized one. We refer the interested reader to~\cite{white2013pricing} for a complete reference of CDS pricing under the ISDA agreements.

CDS are complex products but thanks to standardization, it is not incorrect to consider them like standard goods, with a quantity, the \textit{notional}, and a price, the \textit{spread}. In the rest of this document, we will use call \textit{price} the spread of a CDS, computed using the methodology from~\cite{white2013pricing}.

\subsection{Central counterparty clearing}

OTC derivative trading presents a lot of risks. In particular, counterparty risk has been deemed responsible of the 2007 credit risk by some actors. Counterparty risk is relevant for derivatives, as they are contracts implying some future delivery or payment. Hence, if a party defaults (goes bankrupt) prior to the expiration date of the contract, it cannot fulfill its part of the agreement. This is called counterparty risk.

To cover the counterparty risk, it is possible to trade through a CCP. Both sides have to agree that the transaction will be cleared centrally or are compelled by the regulation to do so. In practice, the clearing counterparty becomes the counterparty of both parties: instead of Company A buying from Company B, Company A will buy from the clearing counterparty and Company B will sell to the clearing counterparty.

Therefore, a clearing counterparty always has a null net balance, but is exposed to the counterparty risk of all the actors it has contracted with. Hence, the clearing counterparty will demand collateral in cash or other assets from both sides. This collateral is called margin. Other assets can be accounted as cash when deducting some fraction from their market value ; this fraction is called a \textit{haircut}. Safe assets have very low haircuts, while more volatile assets might have haircut rates as high as 50\%.

Since the 2009 G20 in Pittsburgh, regulators have mandated the use of central counterparty clearing (CCP) for part of the transactions on standardized derivative securities. 

Clearing counterparties can fix themselves the amount of collateral they require using a pricing model. Collateral can be assimilated to a transaction fee: even if it is not burnt, collateral transfer causes an opportunity cost for its owner. Central counterparties are commonly abbreviated as CCPs.

\subsection{EMIR reporting}
Since the enactment of the European Market Infrastructure Regulation\footnote{\url{https://en.wikipedia.org/wiki/European_Market_Infrastructure_Regulation}} in 2013, most EU counterparties trading in the derivative markets have to report all their derivative transactions.

Through the Bank of France, we had access to some of this highly valuable data. In the rest of this document, we call this data the EMIR dataset. The reporting format is detailed in chapter \ref{chap:data}. 

\section{Analysis of the central counterparty clearing market}

The first research question that we aimed to address is the competition between central clearing counterparties. 

Up to now, CCPs had received little attention from the academic community. Most of the debate had been focused on the netting efficiency of CCPs, how the scope of netting between products and bilateral or central clearing affects the quantity of collateral required to mitigate counterparty risk.

In this project, we wanted to investigate the type of competition at stake between CCPs. CCPs sell protection against counterparty credit risk to their members. For that purpose, they collect collateral from them, which is included in the price of the service offered by the CCP. At the same time, they compete for market share against other CCPs.

Furthermore, central clearing has a two-tiered structure. CCPs directly interact with their members, who in turn, offer clearing through CCPs to their clients. It is possible to study competition on these two dimensions: the choice of CCPs between members, and the choice of CCPs by clients.

There is no clear-cut evidence regarding the model of competition between CCPs and the type of competition at stake between CCPs could have an impact on financial stability.
We wanted to use the EMIR dataset to study the choice of CCPs by members and clients. The EMIR dataset contains several class of derivatives (credit, equity, forex, etc...). We wanted to use the credit derivative data (CDS) that had already been partially preprocessed, and also clean the equity derivative (mainly options and future contracts) data. However, the forced telecommuting caused by the exceptional sanitary situation made us technically unable to clean the equity derivative data, as it involved processing large amounts of data remotely without a proper IT architecture. Hence, our results concern only the CDS data.

Our intended methodology was the following. First, estimate the collateral required by CCPs using machine learning models. Then, include the collateral predictions in a discrete model of CCP choice by members.

However, one major challenge is that CCPs calculate collateral for portfolios of transactions instead of single transactions, which makes the price per transaction difficult to compute. After processing the data to reconstitute portfolios, it appeared that the reported collateral requirements were not predictable enough to be used as input to a CCP choice model. We report the results of our analysis and propose some explanations for the mismatch between the data and our models.

\section{Modelling interdealer transactions}

We then switched to a second problem, unrelated to central clearing, but building on the same CDS transactions. The OTC CDS market is two-tiered: there are both dealers (big banks) and clients. Each client exchanges mostly with a very small number of dealers (one or two) while the dealers trade both with clients and between them. The transactions between dealers and clients form a typical two-sided market, a well-studied problem in economics. Instead, we were particularly interested in the interdealer market that is much less studied.

Our problematic can be formulated through the following question: how do dealers choose other dealers to make a transaction with?

We make the following (strong) hypothesis: in each transaction, one dealer chose the other. This hypothesis forces our model to choose an \textit{aggressor} side for each transaction before computing the likelihood. We model this choice by a trainable prior and show it reduces to a typical expectation-maximization problem. Finally, we formulate this optimization problem as a differentiable objective that can be trained by gradient descent.

Drawing upon this model, we apply techniques inspired by the field of interpretability in machine learning to produce a principled framework for analyzing economical datasets.

Our contribution is also theoretical, with the development of two novel concepts: \textit{Razor entropy} and \textit{top-k Shapley values}. \textit{Razor entropy} is a generalization of our principled training objective that brings theoretical guarantees for our likelihood. \textit{Top-k Shapley values} are a solution to the payoff distribution problem in game theory. We show that they provide interesting properties for monotone submodular value functions and can be computed efficiently. 

\section{Other contributions}

In order to extend our analysis and feed our interdealer transaction predictive model some information about the transactions in the rest of the network, we apply the famous node2vec~\cite{grover2016node2vec} algorithm. When doing so, we made 2 realizations: the existing Python implementations are excessively slow, use huge amounts of memory, and the neighbor sampling as described in~\cite{grover2016node2vec} is not scalable as it uses non linear amounts of memory in a preprocessing step (it can also run slower when using a linear amount of memory thanks to a space--time trade-off).

We designed a randomized neighbor sampling algorithm based on rejection sampling that has the same expected running time as in~\cite{grover2016node2vec} but uses linear amounts of memory. We then programmed a novel implementation of node2vec~\cite{fastnode2vec}
that is orders of magnitude more efficient than existing ones, both thanks to our improved algorithm and our use of JIT (just-in-time) compilation.

\section{Acknowledgements}

We are grateful to Prof.\ Demange for her supervision during this work, and PhD candidate Thibaut Piquard for his unfailing support and help on this project. We also thank the whole macro-prudential policy team from the Bank of France, particularly Dr.\  Julien Idier and Aurore Schilte for their helpful comments on this work.
\chapter{EMIR Data}
\label{chap:data}

\section{Data description}

\subsection{Collection process}

\paragraph{EMIR} The EMIR regulation requires every European Union domiciled entity to submit every derivative contract they enter into to a so-called trade repository.
The detail of the reported data is specified by official guidelines \footnote{\href{https://www.esma.europa.eu/policy-rules/post-trading/trade-reporting}{https://www.esma.europa.eu/policy-rules/post-trading/trade-reporting}}. There are several trade repositories and each one records different subsets of data. Fortunately for our analysis, the DTCC trade repository that we get the data from captures the majority of the CDS market~\cite{osiewicz2015reporting}.

\paragraph{Available data} Most of the collected data is not available to the public. The Bank of France, as a regulatory entity, can access data related to the activities of French companies. In this category enters (i) any transaction made by a French institution, even if the other side of the trade is not French and (ii) any transaction of a product linked to a French company. For credit default swaps, this includes not only CDS on French companies (\textit{single names}), but also the much more traded indices. 

\paragraph{Index} An index is a basket of related assets with a public composition. Indices allow to trade a combination of these products, thus simplifying the actions of traders who want to own their combination. An example of such CDS index is iTraxx Europe, which combines CDS on the 125 most actively traded single names over the last period (indices' contents are updated semiannually). Indices are particularly interesting for our analysis as they are widely and frequently traded and we have access to all the transactions. We also have access to aggregated data provided by Bloomberg to track the interdealer OTC price of indices.

As the EMIR dataset is highly confidential and we accessed regulator's data, we are not able to explicit entity names or give out portfolio compositions in this report. However, this limitation does not apply to the aggregated results we present or the analysis we make.

\subsection{Data format}

The EMIR data we used is composed of 3 kind of \textit{reports} provided by DTCC, a trade repository. Those reports are called: \textit{trade activity report}, \textit{trade state report} and \textit{collateral valuation report}. 

\paragraph{Trade activity reports} The trade activity reports are the most granular ones: they contain the detail of every contract modification that are made. However, they are hard to preprocess because of this granularity.

\paragraph{Trade state reports} The trade state reports contain are daily aggregated versions of the trade activity reports that state the content of every portfolio. Because the trade activity reports are produced by the data collector (DTCC), we try to use them when possible, instead of the more complex trade activity reports. Furthermore, they avoid processing the whole history of activity reports by stating the positions of every portfolio.

\paragraph{Collateral valuation reports} The collateral valuation reports are the simplest ones and contain the collateral for each portfolio. Note that initial margins are only one of the different types of collateral that can be collected by CCPs.

\section{Data cleaning and augmentation}
\label{sec:cleaning}

\subsection{Activity report cleaning}

The activity reports contain a lot of superfluous information that are not relevant to our analysis. Indeed, they contain every modification made to a contract, most of which are irrelevant to the study of counterparty choice, like novations, terminations, error corrections and trade compressions.

A major challenge is that the activity reports are bilateral reports, that is one transaction is to be reported by the two sides if they are both European. This can lead to duplicates of transactions. Furthermore, the reporting process is manual for most entities which can lead to transaction information being incorrectly entered and the same transactions to be counted more than once or twice. Another source of error is the reporting date that can be different from the actual transaction date.
Detecting duplicate transactions is a non-trivial problem that needs to be done separately for each derivative class.

The activity reports have been cleaned for credit derivatives (CDS) in a previous work \cite{demange2020market} in order to extract operations that are \textit{price relevant}, as most of the operations that are contained in the activity reports are back-office transactions such as compression and clearing that do not involve the bargaining of a new price.

This cleaning process used graph matching to de-duplicate transactions, and a lot of handcrafted rules to account for frequent reporting errors. For example, a frequent source of outlier data points was Japanese Yen amounts reported with a wrong or missing currency field, thus multiplying values by $\sim$100 when considered as dollars or euros. A lot of other outlier detection rules that we do not detail here are specific to CDS contracts and allow to correct or discard erroneous lines.

\subsection{Report merging}
\label{sub:merging}

In theory, the collateral valuation report should contain all the information about portfolio initial margins when one counterparty is French, as we accessed the data through the Bank of France. However, when one counterpart is not French, the initial margins will only appear in trade state and trade activity reports.

Thus, a significant amount of efforts was dedicated to processing state reports to augment the content of the valuation reports from the content of the state reports for all kind of assets.

Because the rest of our analysis is specific to credit derivatives, we restricted ourselves to portfolios containing only credit derivatives as it would be harder to analyse only partial portfolios.

Then, we merged those augmented valuation reports with the cleaned state reports to have for each portfolio its contents and initial margin.

This aspect of our work involved a lot of code optimization to make the processing reliable and scalable. Indeed, we ran our code on a $~$200 GB sample covering the last year and it is now used on the data received daily by the Bank of France.

\subsection{Portfolio linkage}

During our data processing, we encountered some situations where the same transaction was reported by both sides, but with different portfolio IDs.

We adopted an approach from graph theory to solve this problem. We consider each portfolio ID as a node and when a transaction is reported with two different portfolio IDs, we draw an edge between those two portfolio IDs.

\begin{figure}
  \centering
  \includegraphics{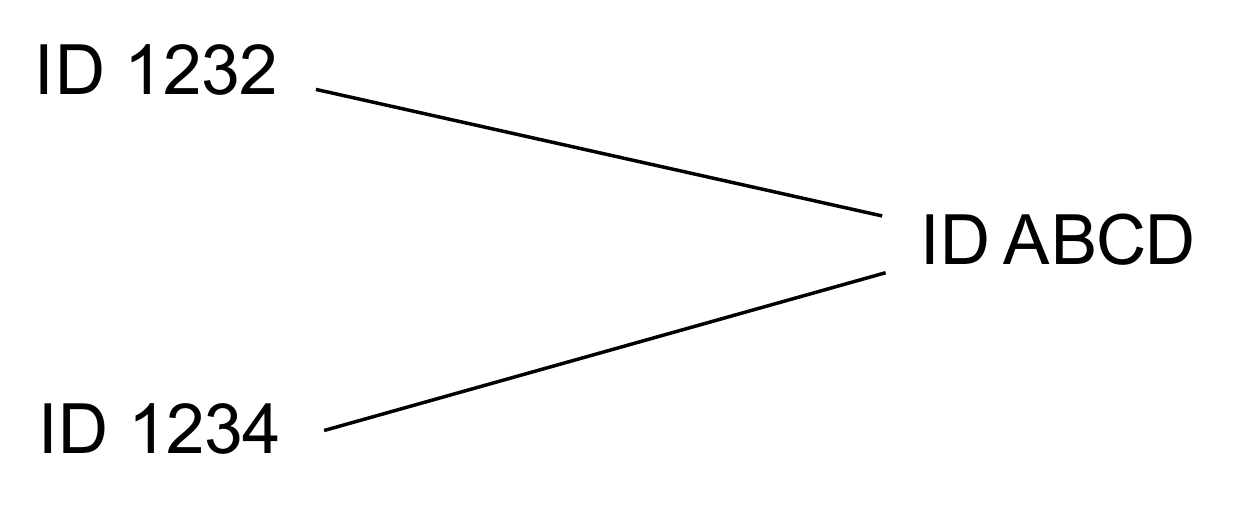}
  \caption{At least two transactions have been carried on. In the first one, one side reported portfolio ID 1234 and the other ABCD. In the second one, one side reported portfolio ID ABCD and the other one 1232.}
\end{figure}

A natural solution is to consider connected portfolio IDs between two counterparties as the same portfolio or, put simply, find the connected components in the undirected graph we defined. We use a disjoint-set data structure to track the components while reading the edges in an online fashion.

\begin{algorithm}[H]
\SetAlgoLined
    \KwResult{Lists of equivalent portfolio IDs }
    
    Initialize a disjoint-set data structure with functions $Union$ and $Find$\;
    \ForEach{transaction in reports}{
        $Union(transaction.portfolio\_id1, transaction.portfolio\_id2)$\;
    }
    Initialize a mapping from string to list of strings\;
    \ForEach{portfolio\_id}{
        Add $(portfolio\_id, Find(portfolio\_id))$ to the mapping\;
    }
    Return the values of the mapping\;

 \caption{Portfolio linkage}
\end{algorithm}

This algorithm has $\mathcal{O}(n)$ space complexity and $\mathcal{O}(n \alpha(m))$ time complexity, where $n$ is the number of nodes, $m$ the number of edges and $\alpha$ the inverse Ackermann function.

After this linkage, we have to recover the characteristics of the portfolios. The contents are easy to define, as they are simply the aggregated transactions. The initial margins are harder to characterize. The collateral valuation report gives us margins for each separate portfolio. To decide on the margin of the merged portfolio, we used voting on the aliased portfolios.

\section{Exploration and visualizations}

\subsection{Comparison of data sources}

We report in Figure \ref{fig:totalmargins} the number of portfolios retrieved from the merging step described in Subsection \ref{sub:merging}. Our processing allowed us to get much more data.

\begin{figure}
    \centering
    \includegraphics[width=.9\textwidth]{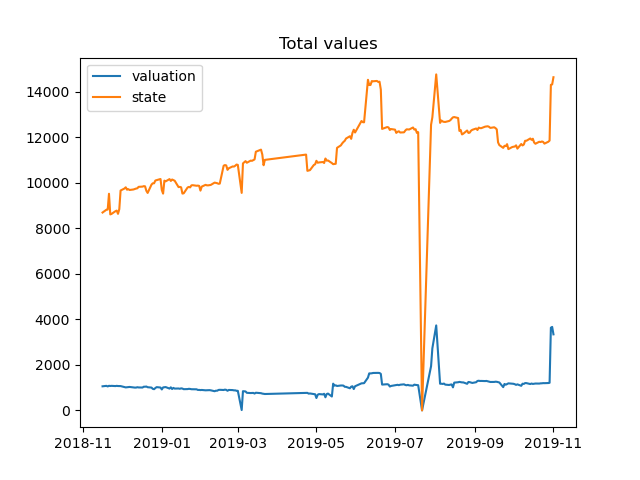}
    \caption{Number of portfolios retrieved from the original data (valuation reports) and from the post processing (state reports)}
    \label{fig:totalmargins}
\end{figure}

However, as we see in Figure \ref{fig:nonzero}, there are more portfolios reported without a margin in the state reports.

\begin{figure}
    \centering
    \includegraphics[width=.9\textwidth]{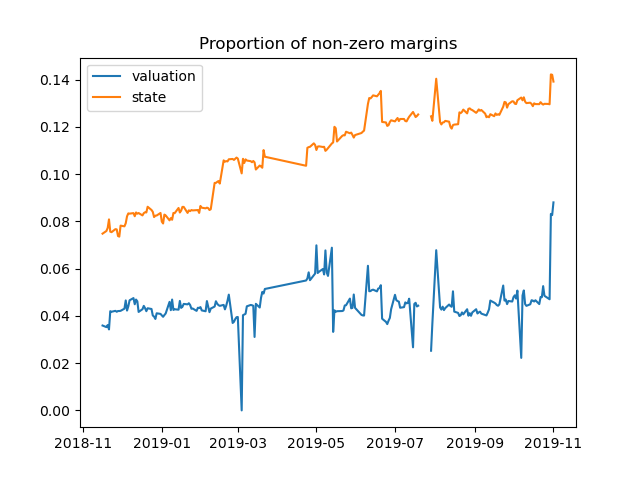}
    \caption{Proportion of zero margins}
    \label{fig:nonzero}
\end{figure}

Finally, we plot in Figure \ref{fig:quartiles} the quartiles of the distribution of nonzero margins to better represent the difference between the two sources. We observe that the margins reported in valuation reports are much larger.

\begin{figure}
    \centering
    \includegraphics[width=.9\textwidth]{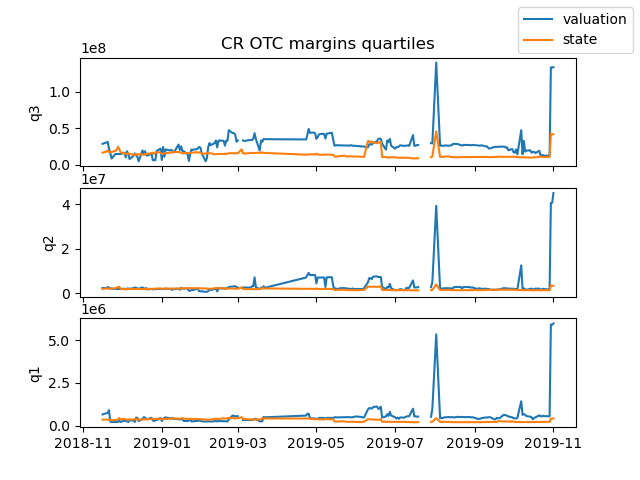}
    \caption{Quartiles of non-zero margins}
    \label{fig:quartiles}
\end{figure}

\subsection{Margin flows}

To show the coherence of the data with known market practices, we plotted the amounts received by each sector and transferred between sectors in Figures \ref{fig:imreceived} and \ref{fig:imflows}. The classification of actors was retrieved from an internal database of the Bank of France produced by hand, hence the large proportion of \textit{Unknown} actors, mostly small.

Unsurprisingly, G14 dealers (the largest 14 dealers) and CCPs are the main entities receiving margin after netting, while Banks and Funds are the providers. Interestingly, huge amounts of margins received by CCPs come from \textit{Unknown} institutions, indicating a fat tail.

Margins transfers from Banks and Funds to CCPs clearly follow a tiered process.
G14 dealers collect initial margins from banks and funds and transfer a part to CCPs. Banks also transfer directly to CCPs.

Finally, these graphs show the instability of the collected data, part of which corresponds to real-world phenomenons (e.g. periodic market events) while inconsistencies in the data collection process are probably to blame as well.

\begin{figure}
    \centering
    \includegraphics[width=\textwidth]{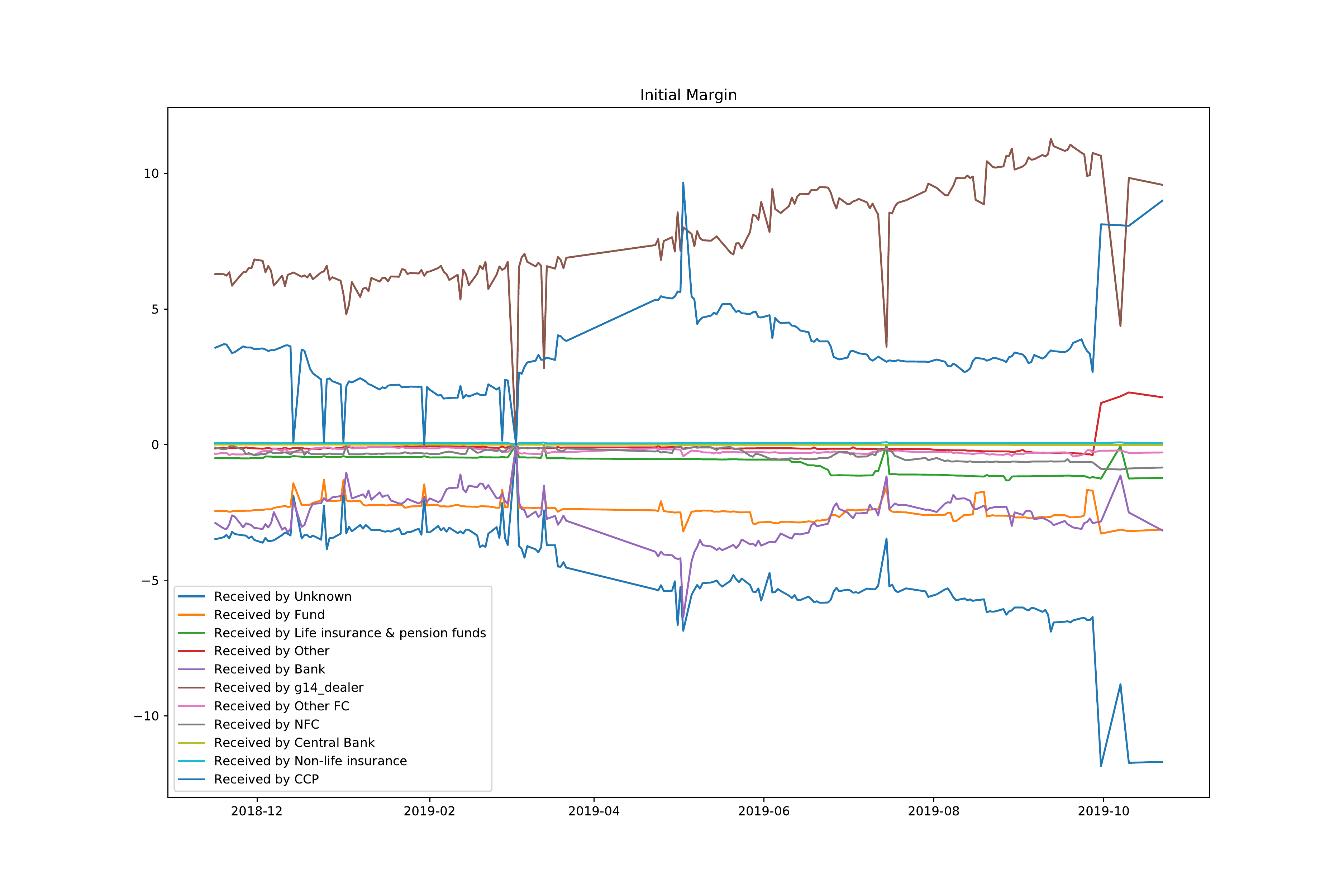}
    \caption{Net margin received by sector (in billion EUR)}
    \label{fig:imreceived}
\end{figure}

\begin{figure}
    \centering
    \includegraphics[width=\textwidth]{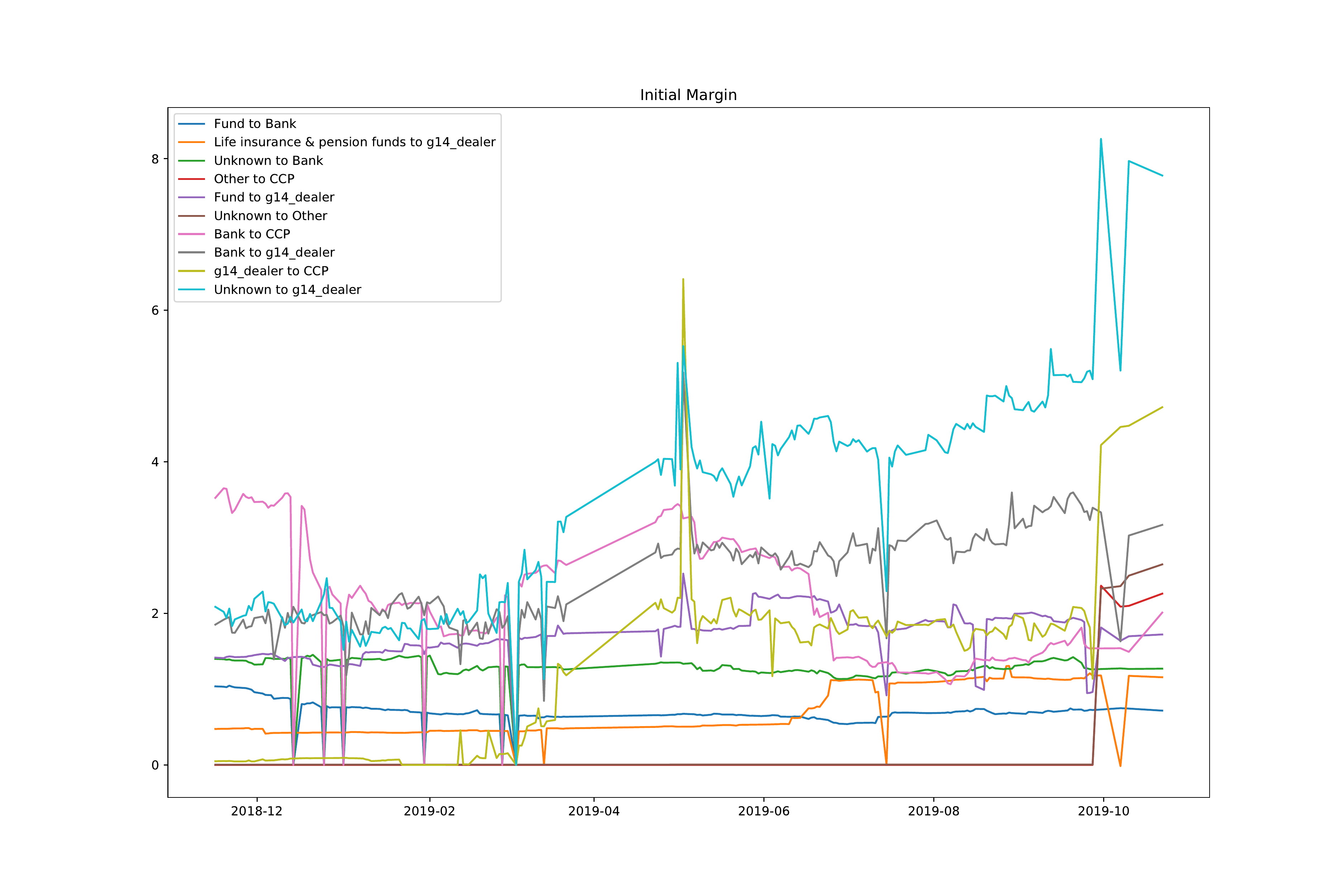}
    \caption{Net margin transfers between sectors (in billion EUR)}
    \label{fig:imflows}
\end{figure}

\chapter{Initial Margin prediction}

\section{Description}

Competition between CCPs is an important question related to financial stability. The main reason to use central clearing is to insure each counterparty against the default risk of the other. However, there are debates on the efficiency of CCPs to reduce systemic risk, as there is a possibility for the CCP itself to default, therefore exposing all its members to a greater risk.

\paragraph{Collateral requirement} To mitigate the default risk of their members, CCPs collect collateral. There are several ways collateral can be required by a CCP, examples being contributions to the default fund,  variation margins that cover position changes and initial margins. In this work, we are only interested in initial margins. Initial margins are an amount of collateral requested by the CCP to cover the liquidity risk of the portfolio. If one side of the contract were to default, the CCP would try to liquidate its position to nullify its market exposure. However, by doing so, the CCP would be subject to additional costs due to the market volatility. Hence, the initial margins are collected to cover those costs and are adjusted with the market conditions, typically the product volatility.

The study of collateral requirements is of great interest as it relates to the netting efficiency and the competition between CCPs.

\paragraph{Netting} Netting is a practice that cancels the exposure towards a CCP of institutions that present positions of equal value in two distinct directions. For example, if institution A bought 15 units from institution B and sold 10 to institution C, and cleared those transactions through a CCP, the net exposure of A to the CCP will be a long position of 5 units. In this example, netting reduces the initial margin requirements of institution A because its position is reduced. Netting can also be performed on multiple correlated products: if two products are highly correlated, the liquidity risk is lower when a portfolio has a short (resp. long) position on the first and a long (resp. short) on the second.

\paragraph{Competition} Therefore, a limited number of CCPs is preferable on one hand as they can capture more transactions and apply netting on a wider scale. On the other hand, a monopoly would be detrimental by possibly raising the clearing fees, thereby reducing the incentives to rely on central clearing. Competition is not necessarily a good thing either, with the risk of a race to the bottom of collateral requirements that would increase the systemic risk.

To study how important margins are in the eyes of clients choosing CCPs, one has to know how much each CCP charges.
However, each transaction is associated with only one CCP, the one that was chosen.

Our goal was to predict the initial margin requirements by the different CCPs and use these estimations as part of a CCP choice model. However, after thoroughly cleaning the EMIR data, we observed that our estimations were not precise enough to yield clear insight on the pricing of CCPs, thus compromising our original idea of using them in a CCP choice model.

\section{Related work}

As the prediction of margins was intended to be part of a competition analysis between CCPs, we divide this section in two parts, one about competition analysis and one about models for initial margins.

There is a large literature related to competition analysis, both between CCPs or between exchanges~\cite{cantillon2008competition}.

\subsection{Competition analysis}

An early paper on the question of competition between CCPs is~\cite{zhu2011there}. They examine manually the changes in fees of the major CCPs and draw 3 conclusions: competition led to tariff cuts, there is no fit-for-all risk management mechanism and there is no evidence that competition among CCPs has led to a deterioration in the robustness of CCPs’ risk management. They based their analysis on the 3 major CCPs operating on the pan-European equity market (LCH.Clearnet SA, EMCF and EuroCCP). It is not clear that the observed reduction of  transaction-based fees is due to competition only, as trading patterns also evolved with algorithmic trading to result with more low-value and high frequency trades. We remark that the greatest risk for CCPs is the replacement cost risk and the risk analysis presented in the paper is focused on CCPs for equity. There is a main difference between the clearing of equities and derivatives due to the insurance provided by the CCP against counterparty risk: in the case of equities, the risk is very limited because the period where the CCP bears the counterparty risk is very short. This explains why the exchange and clearing functions were usually  fully integrated. In the case of derivatives, the CCP takes long time risk.

\cite{park2016empirical} analyse the margin changes in a dataset of futures margins set by CME Group. They first want to model the margin changes after a volatility change. They also want to measure the impact of competition on margin changes.
Since margin changes are infrequent (one every few months), their model asserts that a margin changes only if an indicator rises above a certain threshold level.
They establish that variations of margin follow a EWMA estimator of volatility for futures, an effect that is particularly noticeable on currency futures. They also show that there is a negative correlation between the probability of margin shortfall and the average number of days between margin changes across different futures contract panels.
The authors first fit a linear model to predict the margins from some features.
Their second model to predict the margin changes is more original as they use a Tobit regression, i.e. a regression where the observed value of the dependent variable is censored. This model draws from the intuition that clearing counterparties want to keep human-readable fees and will update their tariffs only when the difference goes above some threshold levels that they set to the smallest historical margin changes. Their third model is a trichotomous Tobit regression (outcomes -1, 0 and 1) that separates the features in positive and negative components to investigate the symmetry of margin changes.
Like many papers in econometrics, they interpret the regression coefficients as feature importance, an issue we tackle in the next chapters.

\subsection{Margin estimation}

There is also a quite well developed, albeit old, literature on margin estimation.

\cite{knott2002modelling} provides an extended review, but the margin definitions evolved since 2002. A more recent treatment is made in~\cite{lopez2017comargin}, where the two major existing frameworks, Value-at-Risk (VaR) and SPAN are detailed. 

\paragraph{SPAN} SPAN was introduced in 1988 by the CME exchange. SPAN makes use of numeric methods to simulate 16 scenarios and assess the variation of the portfolio value. However, SPAN only sums the risks over the different assets contained in the portfolio and does not make use of correlations between assets. The complexity and innovation of the SPAN system comes from the wide range of contracts it can support, with many contract parameters, CCP charges and risk levels that can be chosen by the user.

\paragraph{Value-at-Risk} The VaR margin is the default measure for the aggregate risk exposure and  regulatory capital requirements of banks. VaR can also be used to
set margins on derivatives.
\begin{definition}
The VaR margin $M$ on an asset with return $V$ for risk level $\alpha \in[0,1]$  is defined by the equation:
$$\Pr[V < -M] = \alpha$$
\end{definition}
\cite{lopez2017comargin} extends the VaR margin framework to encompass risks at the CCP level instead of the portfolio level, that is, compute the joint probability of two portfolios exceeding their margins.

\section{Experiments}

After collecting portfolio contents as described in chapter \ref{chap:data}, we tried to predict the initial margin according to some baseline models.

\subsection{Baselines}

The initial margin covers the liquidity risk, hence only depends on the possible price variation during an attempt to nullify the position.

If one supposes that the price of CDS indices follows a normal probability distribution on a given time interval, then their difference will be a centered normal distribution. Furthermore, we have the following property:

\begin{lemma}
    \label{lemma:vol}
    For a given risk level $r \in [0, 1]$, there is a positive constant $\alpha \in \R^+$ such that for any asset following a normal price distribution with standard deviation $\sigma$, its VaR initial margin $M$ is given by:
    
    $$M = \alpha \cdot \sigma$$
    
    $\alpha$ has a simple expression following from the inverse cumulative distribution function of the standard normal distribution.
\end{lemma}

We present some baselines below, in increasing order of sophistication.

A portfolio $pf$ is a set of positions $(pr, q)$ where $pr$ is the product and $q$ is the quantity.

Because we do not have full access to the portfolio contents of all actors, we restrict our analysis to portfolios of French entities. When running the same baselines on the full dataset, the results are much less convincing because of the products that are not observed. As a byproduct, we ignore the portfolios retrieved from state reports.

\subsubsection{Product-oblivious baselines}

The first two baselines are:
$$\widetilde{IM}_1(pf) = \alpha \left|\sum_{(pr, q) \in pf} q\right|$$
$$\widetilde{IM}_2(pf) = \alpha \sum_{(pr, q) \in pf} |q|$$

$\widetilde{IM}_1$ estimates the margin as proportional to the net sum of notionals. This assumes that all products are perfectly correlated and have the same volatility.

$\widetilde{IM}_2$  uses the gross notionals. This is similar to the SPAN system that assumes some kind of independence between products. We furthermore assume that all products have the same risk.

To estimate $\alpha$, we can use a least squares regression:
$$\argmin_\alpha \sum_pf (IM(pf) - \widetilde{IM}(pf))^2$$

However, this gives poor results because the notionals are very imbalanced and do not follow a normal distribution.

Instead, one can assume log-normality and apply log-log least-squares regression~\cite{benoit2011linear}:
$$\log \widetilde{IM}_1(pf) = \log \alpha + \log  \left |\sum_{(pr, q) \in pf} q \right|$$
$$\log \tilde{IM}_2(pf) = \log \alpha + \log \sum_{(pr, q) \in pf} |q|$$

This model can simply be estimated by subtracting the means of the dependent and covariate.

\subsubsection{Sum of volatility baseline}

One can take the volatility of products into account. Applying lemma \ref{lemma:vol} and supposing that risks are measured separately for each products and summed, we establish baseline 3:
$$\widetilde{IM}_3(pf) = \alpha \sum_{(pr, q) \in pf} |q| \cdot \sigma(pr)$$

The values of the standard deviations $\sigma(pr)$ are not estimated but computed using a rolling window of daily prices retrieved from Bloomberg for the last 1000 business days. We restrict the products to indices and hence ignore the portfolios containing less than 80\% of indices (in notional).

We also estimate $\alpha$ with log-log least squares regression.

\subsubsection{Variance-covariance VaR}

Here, we suppose that products follow a Gaussian multivariate distribution. From the properties of multivariate Gaussians:
$$Var(\sum_i X_i) = \sum_i \sum_j Cov(X_i,  X_j)$$

Hence, we propose the following model:
$$\widetilde{IM}_4(pf) = \alpha \sqrt{\sum_{(pr_1, q_1) \in pf} \sum_{(pr_2, q_2) \in pf} q_1 \cdot q_2 \cdot Cov(pr_1, pr_2)}$$

Again, we compute $\alpha$ to minimize the least squares error in log space.

\subsubsection{More complex models}

We also tried more complex models that implicitly compute the VaR without knowing the deviations $Cov(pr_1, pr_2)$. Hence, it would possible to estimate the VaR of portfolios that are not mainly composed of indexes. However, for $n$ products, such a model would have to estimate $\mathcal{O}(n^2)$ coefficients, which is much larger than our number of observations.

We used the same formula as model 3: 
$$\widetilde{IM}_3(pf) = \alpha \sum_{(pr, q) \in pf} |q| \cdot \sigma(pr)$$
but tried to estimate $\sigma(pr)$ instead of giving it as an input. We cannot use the same log-log criterion as there are no efficient solvers.
This gives baseline 5:
$$\widetilde{IM}_5(pf) = \sum_{(pr, q) \in pf} |q| \cdot \alpha_{pr}$$

A direct least squares regression approach led to absurd solutions with negative coefficients. We tried non-negative least squares (using \texttt{scipy.optimize.nnls}) to ensure $\alpha_{pr} > 0$ and constrained least squares (\texttt{scipy.optimize.lsq\_linear}) to also ensure $\alpha_{pr} < 1$ so that that the margin for a product is never larger than the notional.

However, we could not make them converge because there are too many different products compared to the number of observations we had. Furthermore, since we use time-series data that track the content and margin of portfolios over time, many data points are highly correlated.

\subsection{Results}

We restricted ourselves to a subset of the dataset where each portfolio has at least one French member, and where indices represent more than 80\% of the portfolio value.
This resulted in $3115$ unique data points for $65$ different portfolios.

We report the $R^2$ coefficient of the log-log least-squares regressions.

\begin{center}
 \begin{tabular}{|c c|} 
 \hline
 Method & $R^2$ \\
 \hline\hline
 $\widetilde{IM}_1$ & 0.466 \\ 
 \hline
 $\widetilde{IM}_2$ & 0.598 \\
 \hline
 $\widetilde{IM}_3$ & 0.608 \\
 \hline
 $\widetilde{IM}_4$ & 0.500 \\
 \hline
\end{tabular}
\end{center}

We plot the data and estimates of the best model ($\widetilde{IM}_3$) in Figure \ref{fig:margin}.

\begin{figure}
    \centering
    \includegraphics[width=\textwidth]{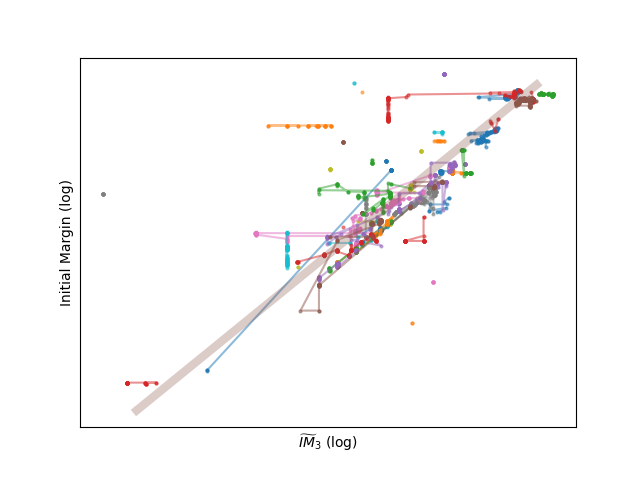}
    \caption{Initial margins versus sum of volatility. Each portfolio is represented by a different curve indicating its \textit{trajectory} over time. The estimate is represented by a straight strip. To protect the confidentiality of the amounts, we hid the scale. Here, $\alpha = 1.022 \cdot 10^{-3}$.}
    \label{fig:margin}
\end{figure}

We observe that although most portfolios fit our model well, there are some outliers that are either isolated or present some horizontal or vertical component. Vertical trajectories mean that the initial margin changes whereas the the volatility of the portfolio did not. They can be associated to compression and netting done by the CCP. Horizontal trajectories mean that the content of the portfolio changed significantly but the margin did not. The most probable explanation is that there is a delay in the update of margins or in the update of the reported values.
\chapter{Transaction graph modelling}

In this part, we are trying to model interdealer choices to understand their dynamics. Our data is a temporal graph with data on the edges. In mathematical terms, we have a set of nodes $V$ and a list of transactions $$(t, a, b, c) \in T \times V \times V \times D$$

$T$ describes temporal information, it can be described continuously with $T = \R^+$ or discretely with $T = \N$. $D$ describes the data of the transaction, i.e. the product, quantity (notional) and price (spread).

Furthermore, the nodes $V$ can be associated some other metadata $M$ from other sources. An example of such data is the spread of a CDS on the dealer (thus, a so-called single name CDS). Since the dealers are all important financial companies, they all have CDS on their debt.
Note that we study indices, not single name products, hence we do not take the exact index composition as a parameter of our model.

\section{Related work}

\subsection{Counterparty choice}

The closest work to ours is~\cite{du2019counterparty}, as they also study counterparty choice on the OTC CDS market.
They motivate this question in regard of the 2008 crisis, during which Bear Stearns and Lehman Brothers would not find any counterparties to trade with them because their risk measure was too high.

\paragraph{Data}~\cite{du2019counterparty} also base their analysis on confidential transaction level data provided by DTCC. One difference is that we don't have access to the same subset of the dataset, as they accessed it through the FRB and we accessed it through the Bank of France.  Hence, our respective analyses have the same restrictions, namely that we can only access transactions for which either one side or the reference entity of the contract is regulated by the institution we obtain the data form. They based their analysis on single-name CDS contracts, while we focused on the CDS indices that represent the largest part of our dataset.

\paragraph{Modelling} The most crucial difference is that they model the dealer choice by clients, including non-dealers. Another difference is that their model predicts the choice of a seller dealer by a buyer. This hypothesis makes less sense on the interdealer market where intermediation is a driver of transactions~\cite{li2019dealer}. On the opposite, our model follows a maximum likelihood approach to attribute a direction to the choice (sometimes the choice is made by the buyer, sometimes by the seller).

\paragraph{Methodology} On the algorithmic part, they only make use of logistic or multinomial linear regressions whereas we use multilayer perceptrons and show they (unsurprisingly) better fit the data.\\
Their results are purely in-sample, meaning they do not test or validate their predictions with data different from their training set. On the other hand, we temporally split our data and make use of Bayesian hyperparameter optimization and several regularization techniques to ensure our model learns \textit{real} patterns.\\
Finally, their conclusions are based on the amplitude of the regression coefficients whereas ours use interpretability and causality techniques framed in information theory to establish more rigorous measures of feature importance.

\subsection{Graph modelling}

Our methodology is also closely related to the wider literature of link prediction and (temporal) graph modelling.

\subsubsection{Link prediction}

The survey paper~\cite{liben2007link} defines link prediction with the following question: \textit{given a snapshot of a social network, can we infer which new interactions among its members are likely to occur in the near future?} Our modelling is inspired by this question but is different on several points: (i) we have multigraphs, i.e. two entities can be linked by more than one edge (ii) the edges contain information. In this framework, any possible pair of nodes is associated a connection weight that describes some sort of similarity between them but nothing is said about metadata.

\cite{al2006link} showed that supervised learning was a possible solution: classifiers can learn to discriminate between fake and true links.
To take care of the edge information, one could give the edge information as input when trying to predict existing edges. However, it is not clear how one could generate plausible the metadata of fake edges.

\paragraph{GAN} Some approaches used the famous framework of Generative Adversarial Networks~\cite{goodfellow2014generative} for link prediction in dynamic networks~\cite{lei2019gcn} or representation learning~\cite{wang2018graphgan}. However, our final goal is to draw insights about the behavior of actors and it is not straightforward whether a discriminator would really learn what we want. To give an example, suppose that we add totally uncorrelated inputs that are hard to generate, like high resolution images. Then the discriminator might base its judgement on the uncorrelated inputs rather than the legitimate ones. This remark motivates us to predict part of the data, and we justify this with information-theoretic arguments in section \ref{sec:ait}.

\subsubsection{Graph modelling}

\paragraph{Representation learning }A lot of papers are concerned with learning representations of either graph with edge features on one hand and temporal graphs (also called dynamic networks) on the other. Although representation learning is not the primary motivation of our approach, our method produces entity embeddings as a byproduct.

\cite{gong2019exploiting} proposes an enhanced graph neural network~\cite{scarselli2008graph}. They extend the standard graph convolutional neural network~\cite{kipf2016semi} to take multi-dimensional edge features into account.\\~\cite{goyal2018embedding, goyal2018capturing} (two similar papers by the same authors) use an auto-encoder to reconstruct edge attributes from node attributes.

In~\cite{singer2019node}, the authors develop an alignment technique for node2vec~\cite{grover2016node2vec} embeddings using orthogonal transformations, hence producing unsupervised node representations that are consistent over time.\\~\cite{xu2020inductive} allows to produce inductive representations of nodes (which means one can predict representations for nodes outside of the training set) and use time as a model input (using what they call functional time encoding).\\
From the model descriptions and experiments performed in the literature, we conclude that although the formulation looks similar, temporal graph learning does not seem to be adapted to our problem as transactions are more similar to events than some graph snapshots.

\paragraph{Interpreting graph convolutional networks} Finally, a paper of interest is~\cite{li2020explain}. They apply graph neural networks with edge weights to networks of financial transactions, making node classification experiments on the Bitcoin network. They present several interpretability methods for both informative components detection and node feature importance. In particular, one of their methods optimizes a constrained multiplicative mask applied the neighboring edges of a node to minimize the cross-entropy of the prediction.\\
It is the only reference we found on information theory and model interpretability applied  to the analysis of transaction networks.

\section{Motivations}

In this section, we explicit our problem in more details.

As stated above, our data is tabular with both continuous and discrete features. In particular, 2 categorical features are of special interest: $buyer$ and $seller$, because they are directly linked to economical concepts of competition and counterparty choice.

The fundamental question we are trying to answer is: \textit{How do transactions happen, and why?}\\
This model-independent question is intrinsically linked to the field of causality. However, we treat it with techniques from the field of model interpretability, as most of the existing literature confuses those two different concepts, and using arguments developed in \ref{sec:ait}.

The current literature related to transaction graph modelling is somewhat limited, as the closest community efforts are concerned with the weakly related concept of temporal graphs and more general benchmarks of node classification exposed in the previous section.

The starting point of our reflection is that our model should be conditional on the transaction data. Indeed, different products lead to very different transaction graphs. In this work, we are not interested in what products people trade, nor the quantities or prices, as it would probably involve studying global market trends. Instead, we put this information as an input.

We also assume that the choice to trade is ultimately made by one of the two counterparties. This hypothesis is unrelated to the classical economical question of whether trades are zero-sum (hence implying that only one of the two parties will benefit from it). We don't assume that the choice \textit{direction} is linked to any concept of benefit, e.g. that the one who chooses the other does it because they know they will benefit from the transaction.\\ Instead, we apply Occam's razor: \textit{``Accept the simplest explanation that fits the data''} (chapter 28 of~\cite{mackay2003information}).

We show in the next section that modelling the choice direction as a superposition is the same as choosing the directions that maximize the likelihood.\\
Ultimately, our derivations answer the question \textit{who chooses whom?} with \textit{the side that is the most likely to choose the other}.

Thus, we make a very strong assumption in our modelling efforts for the choice dynamics, but that is coherent with current practices in economics. Indeed, in imbalanced markets, like the client-dealer market studied in~\cite{du2019counterparty}, people usually model the choice that has the smallest entropy. On the client-dealer market, there are much more clients than dealers and it is widely assumed that clients choose their dealers. We relate this practice to our Occam's razor -- maximum likelihood principle and extend it to the interdealer market.

\section{An unsupervised objective for transaction graphs}
\label{sec:objective}

\subsection{Notations}

In this section, we use the following notations and conventions:
\begin{itemize}
    \item $t$ is the time, that is naturally associated with a filtration on the data items, which we will use in our validation process to split the data chronologically.
    \item $a$ and $b$ are the two sides of a transaction  that don't need to be symmetrical ($a$ can be the seller and $b$ the buyer)
    \item $c$ is the context of the transaction, e.g. the product, price and quantity. It can also include features that are computed from $t$, like the week day, or data from previous transactions.
    \item $c_a$ (resp. $c_b$) is the context from the perspective of $a$ (resp. $b$). We will instead write them as $c$ when they are present in conditional probabilities:
    $$\Pr[\dots | a, c] := \Pr[\dots | a, c_a]$$
    \item The notation $\Pr[\dots | a, c]$ is deliberately the same as $\Pr[\dots | b, c]$, i.e. we reflect in our notation that we consider a unique probability distribution instead of one for each side (buyer and seller). A consequence is that the bit of information indicating whether we are modelling the buyer's or seller's choice has to be included in the counterparty context $c_a$ or $c_b$. 
\end{itemize}

\subsection{Likelihood function}

Hence we are left with characterizing for each transaction $(a,b,c)$ the likelihood of the transaction.

If we consider that $a$ is the chooser (we write $a \rightarrow \cdot$), the likelihood of the transaction is naturally defined as
\begin{align*}
    L[a, b | c, a \rightarrow \cdot] &:= \Pr[b | a, c, a \rightarrow \cdot]\\
    &= \Pr[b | a, c]
\end{align*}

The first equality is the assumption of our model. The second equality comes from the fact that $a \rightarrow \cdot$ is not an observable event of our data but our interpretation of it. Hence, $a \rightarrow \cdot$ is not part of the data but an \textit{a posteriori} conception of the observer.

Likewise, the likelihood of the transaction given that $b$ chooses (we write this  event $b\rightarrow \cdot$) is $L[a, b | c, b \rightarrow \cdot]$ = $\Pr[a | b, c]$. 

The likelihood function $L$ is not \textit{sticto sensu} a probability but can be thought of as one. In particular, we define by the law of total probabilities:
$$
    L[a,b | c] := L[a,b | c, a\rightarrow \cdot] \Pr[a\rightarrow \cdot] + L[a,b | c, b\rightarrow \cdot] \Pr[b\rightarrow \cdot]
$$

\subsection{Posterior maximization}

$\Pr[a \rightarrow \cdot]$ and $\Pr[b \rightarrow \cdot]$ are part of our modelling, so we can rewrite them as coefficients:
$$
\left\{
\begin{array}{l}
    w_{a\rightarrow \cdot} := \Pr[a \rightarrow \cdot]\\
    w_{b\rightarrow \cdot} := \Pr[b \rightarrow \cdot]
\end{array}
\right.
$$

Following from them being probabilities of exclusive (virtual) events (by our assumption that one entity chooses the other), we have the constraints
$$
\left\{
\begin{array}{l}
    w_{\{a, b\} \rightarrow \cdot} \in [0, 1]\\
    w_{a\rightarrow \cdot} + w_{b\rightarrow \cdot} = 1
\end{array}
\right.
$$

In a maximum a posteriori estimation setting, one is interested in maximizing the parametric likelihood $L_w$:
\begin{align*}
    L[a,b | c] &= \max_w L_w[a,b | c]\\
    &= \max_w 
    L[a,b | c, a\rightarrow \cdot] w_{a\rightarrow \cdot} + L[a,b | c, b\rightarrow \cdot] w_{b\rightarrow \cdot}\\
   &= \max(L[a,b | c, a\rightarrow \cdot], L[a,b | c, b\rightarrow \cdot])\\
   &= \max(\Pr[b | a, c], \Pr[a | b, c])
\end{align*}

Hence, if we want to approximate $\Pr[\cdot|a,c]$ and $\Pr[\cdot|b,c]$ with a model $f_\theta$, it is natural to maximize the log-likelihood:
$$\argmax_\theta \mathbb{E}_{a,b,c}[\log \max( f_\theta(a, c)[b], f_\theta(b, c)[a] )]$$

In all our applications, $f_\theta$ is a differentiable function and we learn the parameters $\theta$ using gradient descent optimization.

\subsection{Expectation--Maximization}

Another way to see the training of objective with gradient descent is as an instance of the expectation maximization method.

The latent variable are the binary choices $\{a \rightarrow \cdot, b \rightarrow \cdot\}$ for each datum. We define a binary variable $direction$ that represent this choice. Our practical learning objective is now:
$$\argmax_{\theta, direction}  \sum_i \log
\left\{
\begin{array}{ll}
     f_\theta(a_i, c_i)[b_i] & \text{if } direction_i \\
     f_\theta(b_i, c_i)[a_i]  & \text{else}
\end{array}
\right.
$$

Supposing we use gradient descent to optimize $f_\theta$, the forward pass evaluates the model and chooses $direction_i$ for each datum. It is perfectly equivalent because the gradient of the maximum of two outputs is the gradient of the largest so the other output is simply ignored.

The backward pass computes the gradient of the loss according to $\theta$, and the optimizer locally solves
$$
\argmax_\theta \mathbb{E}_{x,y,c}[\log f_\theta(y, c)[x]] = \argmin_\theta - \mathbb{E}_{x,y,c}[\log f_\theta(y, c)[x]]
$$

where $x$ and $y$ are $a$ and $b$ swapped according to $direction$. This objective is exactly a cross-entropy.

Note that unlike the vanilla EM algorithm, we don't fully optimize $\theta$ during the backward pass and re-evaluate $direction$ at each forward pass.

\section{Modelling and optimization}

In this section, we describe the models and optimization process.

\subsection{Models}

In all our experiments, we use multilayer perceptrons.

\begin{figure}
    \centering
    \includegraphics[width=\textwidth]{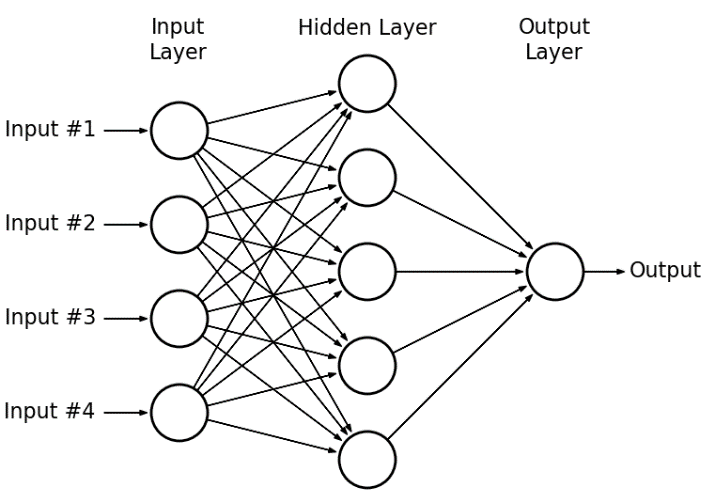}
    \caption{A typical multilayer perceptron with one hidden layer}
\end{figure}

In our experiments, we use between 0 and 2 hidden layers. When the number of hidden layers is 0, multilayer perceptrons become linear models.

We use the ReLU function as activation function, as it is well-studied and one of the most popular activation functions~\cite{nwankpa2018activation}:
$$ReLU(x) = \max(0, x)$$

Since our model has to learn probability distributions, its output is normalized by the softmax function (sometimes more accurately named softargmax):

$$
softmax(z_1, \cdots, z_K)_i = \frac{e^{z_i}}{\sum_{j=1}^K e^{z_j}}
$$

In the above equation, $K$ is the number of classes. In our application problem, it is the number of dealers.

Furthermore, to reflect the structural constraint that a dealer will never trade with itself, we manually set $z_i = -\infty$ for the choosing party during both training and inference.

\subsection{Optimization}

Our models are optimized using Adam~\cite{kingma2014adam, reddi2019convergence} with the default parameters ($\beta_1=0.9$, $\beta_2=0.999$, $\epsilon=10^{-8}$) if not for the learning rate that we fine-tune.
Adam is a gradient-based  optimizer for stochastic functions.

The optimization process is iterative and organized in epochs. During an epoch, the optimizer will be run once on each example of the training dataset. Instead of either optimizing the full objective or running on individual examples, we optimize so-called batches. The batch size is the number of samples that are simultaneously evaluated by the stochastic objective.

The batch size is an important parameter that affects both the convergence speed and the generalization ability of neural networks and can be viewed considered as a regularization technique~\cite{smith2018disciplined}. We also set it as a hyperparameter.

\subsection{Validation}

To follow the best practices~\cite{bengio2012practical}, people should use 4 different random subsets in a layered optimization process, each layer calling the previous one as a subroutine. In our setup, these subsets are

\begin{enumerate}
    \item training set for gradient descent
    \item validation set for early stopping, see the next subsection
    \item hyperparameter validation set
    \item testing set
\end{enumerate}

These practices are adapted when the quantity of data is sufficient to have a low variance relative to the random sampling of those subsets.

For smaller quantities of data, it is possible to apply variance reduction techniques like bootstrapping or k-fold cross-validation.

However, our dataset is chronological. Hence, it contains trends that can lead to overly confident estimations that cannot transpose to real-world evaluations.
For example, the COVID-19 crisis (that we didn't include in our analysis) profoundly transformed the landscape of financial markets. In particular, the prices of Chinese assets dropped in January while they regained value in March. This is an extreme example but similar changes are impossible to predict so we don't want our model to learn them. In a random sampling scenario, the model would be able, even without providing it with dates, to pick up correlations in the data that would boost its performance.

No major event happened during the period we studied, hence such major changes should not appear in our data. However, such changes can happen on a smaller scale which motivates us to split our data chronologically.

The downside of chronological data splitting is that it makes variance reduction techniques much harder to apply without seriously reducing the sizes of training sets.

Furthermore, the quantity of data we have is quite small and splitting in 4 subsets leads to poor performance.

Therefore, we split our data  chronologically in a training and an evaluation set. The training set is used by the gradient descent optimizer and the early stopping procedure while the evaluation set is used to validate the hyperparameter optimization and measure the generalization capabilities.

\subsection{Early stopping}

The objective value as a function of the number of training steps (epochs) has diminishing returns. Furthermore, training for too many epochs results in a phenomenon called over-fitting: the model learns the training set well but performs poorly on the testing set.

Early stopping is a regularization technique that limits the number of epochs. The need to reduce the training times, even at the cost of slightly decreased performance, is motivated by the gains of the hyperparameter search.

The principle of early stopping is to compute the loss value on a validation set at regular intervals and stop the training as soon as the monitored quantity increases, even if some variants exist~\cite{prechelt1998early}.

Because we don't have a validation set, our early stopping is not based on an out-of-sample validation set. Furthermore, we make heavy use of hyperparameter optimization, that was noted by~\cite{bengio2012practical} to conflict with vanilla early stopping.

Instead, our early stopping procedure aims to detect the convergence of the training loss. We define two parameters:

\begin{itemize}
    \item $\alpha \in ]0, 1]$: minimum geometric improvement
    \item $k \in \N^*$: maximum number of epochs without improvement
\end{itemize}

The algorithm is the following:

\begin{algorithm}[H]
\SetAlgoLined
    \texttt{CNT $\leftarrow$ k}\;
    \texttt{BEST $\leftarrow$ inf}\;
    \While{$CNT > 0$}{
        \texttt{train\_loss $\leftarrow$ train\_one\_epoch()}\;
        \eIf{$train\_loss < \alpha ~ BEST$}{
            \texttt{BEST $\leftarrow$ train\_loss}\;
            \texttt{CNT $\leftarrow$ k}\;
        }{
            \texttt{CNT -= 1}\;
        }
    }
 \caption{Early stopping}
\end{algorithm}

One would be tempted to set $\alpha = 1$ to make the algorithm insensitive to order-preserving loss transformations. However, we monitor the training loss that will most likely decrease frequently. Hence we use $\alpha < 1$ to effectively limit the number of training steps.

In our experiments, we use $\alpha = 0.99$. $k$ is set such that $k$ epochs train in a \textit{reasonable} time, where the definition of \textit{reasonable} depends on the time we are ready to wait. In our hyperparameter optimization setting, $k = 50$ gives effective waiting times of a few seconds.

\subsection{Dropout}

Dropout~\cite{srivastava2014dropout} is one of the most popular regularization techniques, and was crucial in improving generalization.

The principle of dropout is to randomly set some neurons in hidden layers to zero during training. Thus, for a network with $n$ neurons, dropout trains a collection of $2^n$ networks with extensive weight sharing. At test time, dropout is disabled and the inference is done by the full network.

In~\cite{srivastava2014dropout}, the authors also apply dropout on the inputs but remark that the ideal retention probability is closer to 1 than for hidden layers. Because the gains are small and we then remove inputs in our interpretability study, we don't apply dropout on the inputs as it could affect the results of our other experiments. 

In our experiments, the dropout rate is uniform across layers and is set as a hyperparameter.

\section{Hyperparameter optimization}

In the previous section, we defined several hyperparameters controlling our model.

Following the results of~\cite{bergstra2011algorithms}, we applied Tree-structured Parzen Estimator (TPE) optimization on the hyperparameter space.

We first describe how Tree-structured Parzen Estimators work, then present the implementation we used and explicit our parameter space.

\subsection{Tree-structured Parzen Estimators}

Let's assume we want to minimize a function $y=f(x)$. However, the arguments of $f$ define a configuration space that cannot be simply expressed as a product of standard sets (like finite sets, $\N$ or $\R$). This setting is exactly the case in hyperparameter optimization, where making a choice on one setting can modify the choices on others. For example, sampling a number of hidden layers modifies the number of layer sizes one has to specify.

In this case, the simplest approach is to sample each parameter independently according to the distribution of previously obtained results.

Some parameters are not present in all function evaluations, hence we only consider for each parameter the function evaluations where it was used.

Parzen estimators~\cite{parzen1962estimation, rosenblatt1956remarks} are non-parametric density estimates that are similar to Gaussian Mixtures. 

Unlike other approaches like Gaussian processes that model the conditional distribution $p(y | x)$, Parzen estimators are used to model $p(x | y)$. Note that Parzen estimators only model the value of one hyperparameter.

Since Parzen estimators are not conditional, we divide the observations in two sets, $L$ and $G$, depending on the objective value $y$. $L$ is a lower quantile and $G$ is an upper quantile. Then, we estimate two densities with Parzen estimators:

$$
\left\{
\begin{array}{ll}
     l(x) &= p(x | y \in L)  \\
     g(x) &= p(x | y \in G) 
\end{array}
\right.
$$

To choose a new value for a hyperparameter, one wants to maximize the expected improvement

$$
EI(x) = \mathbb{E}[\max(0, y^* - f(x))]
$$

In the above equation, $y^*$ is the current best objective value. We also supposed without loss of generality that our objective is to minimize $f$, as it is often the case of loss functions in deep learning.

The computations in~\cite{bergstra2011algorithms} show that $EI$ is approximated as an increasing function of $\frac{l(x)}{g(x)}$. This result is coherent with the intuition that one wants to choose values of $x$ that are much more frequent for good rather than bad objective values.

Hence, the chosen sampling approach is simply to draw a fixed number of candidates according to $l(x)$ and choose the one that maximizes $\frac{l(x)}{g(x)}$.

\subsection{Optuna and define-by-run optimization}

To make full use of the capabilities of Tree-structured Parzen Estimators that sample one variable at a time, one needs to be able to change the parameters that are sampled conditional to the values of the previous ones.

Tree-structured Parzen estimators are an instance of \textit{independent samplers}, that is, they model each parameter separately. However, the hyper-parameter space can depend on the values of already sampled parameters.

Optuna \cite{akiba2019optuna} is a define-by-run hyperparameter optimization framework implementing many state-of-the-art techniques. In particular, its API is simple to use, only requiring one to use a \texttt{trial} object to choose the parameter values.

An advantage is that some parameters can be ignored, e.g. dropout when there are no hidden layers. The full code of our hyperparameter search is outlined in Appendix \ref{code:hyperparam}

\section{Results}

\subsection{Experimental setting}

\subsubsection{Dataset}

\paragraph{Intragroup} Our dataset contains a high number of transactions that are done between entities belonging to the same group. Some entities make a large number of transactions with their own group, thus making our analysis less relevant. Therefore, we ignore those transactions.

\paragraph{Counterparts} We select the counterparts belonging to more than $1/100$ of transactions. This process gives us 19 entities that are part of 12 financial groups: Bank of America Merrill Lynch, Barclays, BNP Paribas, Citibank,
Credit Suisse, Deutsche Bank, Goldman Sachs, HSBC, JP Morgan Chase, Morgan Stanley, Société Générale, Nomura. Compared to \cite{du2019counterparty}, the only two missing dealers in our restricted dataset are RBS Group and UBS.

\paragraph{Products} We select the indices representing more of $1/100$ of transactions: CDX.EM (emerging markets, cdxem), CDX.NA.HY (North America high yield, cdxhy), CDX.NA.IG (North America investment grade, cdxig), iTraxx Australia (itxau), iTraxx Europe main (Europe investment grade, itxeb), iTraxx Europe senior financials (Europe financials investment grades, itxes), iTraxx Europe Crossover (Europe high yield, itxex), iTraxx Japan (itxjp), iTraxx Asia ex-Japan (itxxj).

\paragraph{Train/test split} We make a 80/20 split and obtain 14349 transactions between December 2018 and October 2019. Our training set has 11538 transactions and our testing set 2811.

\subsubsection{Features}

In total, we have 9 features:

\begin{itemize}
    \item \textit{entity}: categorical feature to represent the choosing entity
    \item \textit{direction}: a Boolean indicating whether \textit{entity} is the buyer or seller
    \item \textit{product}: categorical feature representing the traded product 
    \item \textit{notional}: equivalent of quantity for CDS
    \item \textit{price}: spread of the contract, equivalent to a price. We standardize it for each product.
    \item \textit{market price}: average spread of this product on the previous day on the interdealer market. This auxiliary data comes from Bloomberg. We standardize it for each product. As seen on Figure \ref{fig:price}, it is coherent with the values from our dataset.
    \item \textit{Dealer spread}: Array of the spreads for each dealer on the previous day. It indicates how the market perceives the risk of this dealer.
    \item \textit{day}: categorical variable indicating the day of the week
    \item \textit{time}: numerical variable indicating the time of the day normalized between 0 and 1
\end{itemize}

\begin{figure}
    \centering
    \includegraphics[width=\textwidth]{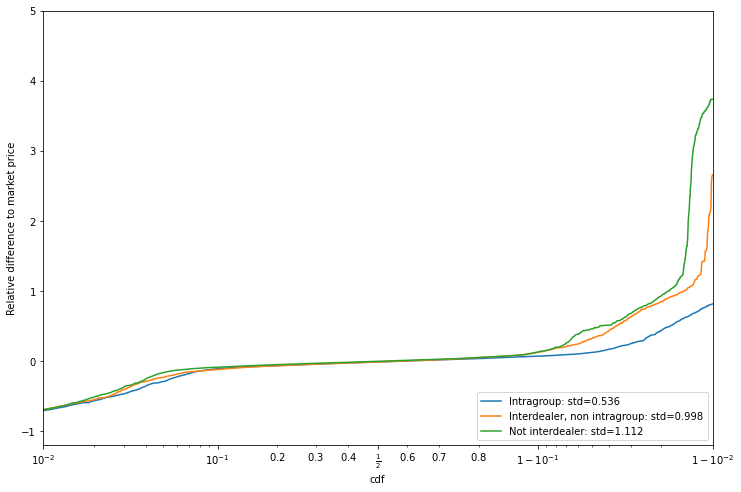}
    \caption{Interdealer market price retrieved from Bloomberg. }
    \label{fig:price}
\end{figure}

\subsubsection{Categorical embeddings}

To handle the categorical variables (\textit{entity}, \textit{product} and \textit{day}), we used embeddings \cite{guo2016entity}. The embedding dimensions are hyperparameters.

\subsubsection{Hyperparameters}

\begin{itemize}
    \item \textit{entity\_dim}: dimension of the \textit{entity} embedding, integer between $1$ and the number of entities
    \item \textit{product\_dim}: dimension of the \textit{product} embedding,  integer between $1$ and the number of products
    \item \textit{day\_dim}: dimension of the \textit{day} embedding, integer between $1$ and $3$
    \item \textit{nb\_layers}: number of layers, between $0$ (multinomial logistic regression) and $2$.
    \item layer sizes: one variable for each layer, between $32$ and $256$ (when $nb\_layers = 1$) or $128$ (when $nb\_layers = 2$). The TPE optimizes their logarithm.
    \item \textit{dropout}: probability of dropout in $[0, 1]$
    \item \textit{use\_batch}: if true, use stochastic gradient, else make updates based on the full dataset
    \item \textit{batch\_size}: if \textit{use\_batch} is true, a power of two between $512$ and $8192$
    \item \textit{lr}: learning rate, follows a log scale between $10^{-4}$ and $10^{-2}$.
\end{itemize}

\subsubsection{Implementation}

We implemented our models using the PyTorch deep learning framework. It provides efficient tensors, automatic differentiation and gradient-based optimizers like Adam.

\subsection{Training curves}

The model appears to successfully learn and generalize. We report a sample training curve in Figure \ref{fig:loss} for the best hyperparameters reported in the next subsection.

\begin{figure}
    \centering
    \includegraphics[width=\textwidth]{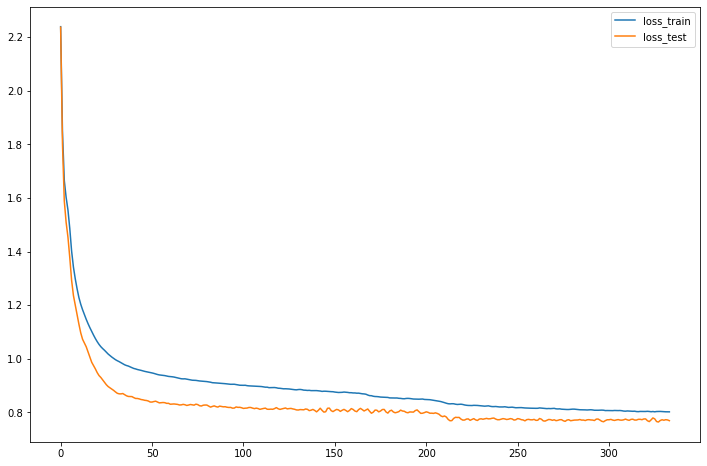}
    \caption{Loss curve for the training and testing datasets. They are correlated and decrease, which indicates that our learning objective is well-defined. The x-axis represents the epochs i.e. full pass on the dataset}
    \label{fig:loss}
\end{figure}

At first glance, it is surprising that the training loss is larger than the test loss. This can be explained both by the small dataset size, statistical artifacts in the split that make the test dataset \textit{easier}, and our extensive use of regularization, with both early stopping and dropout. One can see in Figure \ref{fig:loss_no_dropout} that without dropout, the training loss becomes indeed smaller than the testing loss.

\begin{figure}
    \centering
    \includegraphics[width=\textwidth]{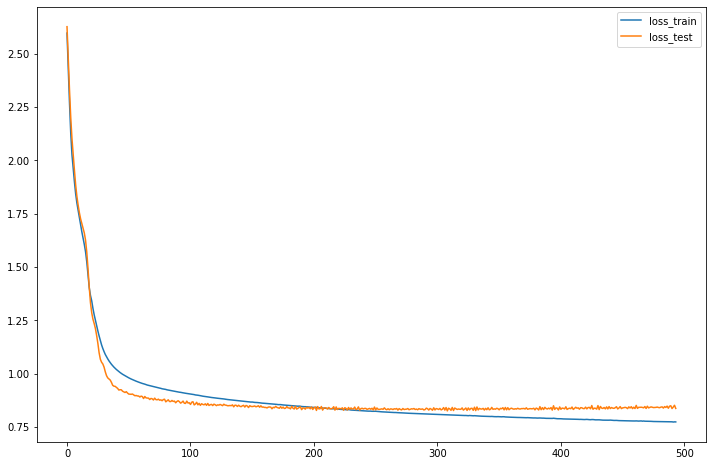}
    \caption{Loss curve for the training and testing datasets with $dropout = 0$. The x-axis represents the epochs.}
    \label{fig:loss_no_dropout}
\end{figure}

\subsection{Hyperparameter importance}

During our hyperparameter optimization process, we explored a lot of different settings. This allowed us to determine what parameters were the most important to improve our metrics.

We report our results in Appendix \ref{chap:hyper-results}.

Based on these experiments, we chose the following values:

\begin{itemize}
    \item \textit{entity\_dim}: $16$
    \item \textit{product\_dim}: $6$
    \item \textit{day\_dim}: $2$
    \item \textit{nb\_layers}: $1$
    \item layer sizes: one hidden layer of size $50$
    \item \textit{dropout}: $0.7$
    \item \textit{use\_batch}: Yes
    \item \textit{batch\_size}: $2^{12} = 4096$
    \item \textit{lr}: learning rate, follows a log scale between $10^{-4}$ and $10^{-2}$.
\end{itemize}

\subsection{Embeddings}

Finally, we remark that our training procedure produces meaningful embeddings for products. To visualize those embeddings, we apply a principal component analysis (PCA) in dimension 2 (Figure \ref{fig:pca}).

\begin{figure}
    \centering
    \includegraphics[width=\textwidth]{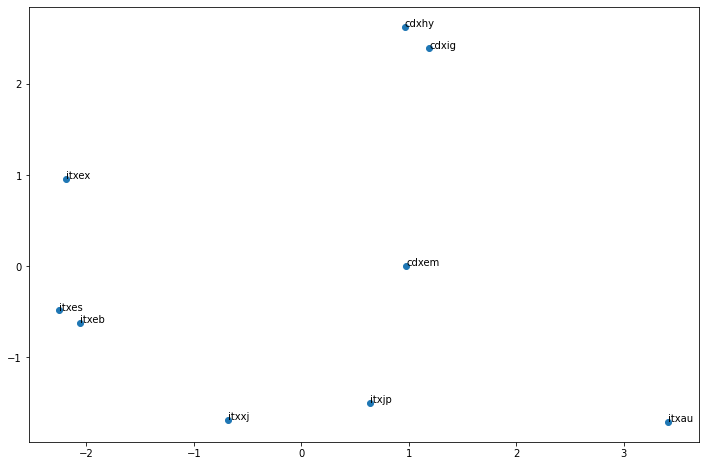}
    \caption{PCA of the products embeddings}
    \label{fig:pca}
\end{figure}

We observe that although the network never saw the characteristics of the products, it learned meaningful embeddings that group similar products together. 
For example, we see that CDX High Yield and Investment Grade are very similar. CDX Emerging Markets is closer to the other CDX indices than to the iTraxx indices. iTraxx Japan (itxjp) and iTraxx Asia ex-Japan (itxej) are grouped together as they both concern Asian markets. Finally, the 3 largest European iTraxx indices are grouped together, and the investment grade indices (itxes and itxeb) are the most similar.

On this OTC market, there are too many products with very low volumes that we ignored. However, these results seem encouraging and would motivate our approach on markets with a fat tail to learn embeddings. The results were not as meaningful for institutions, and because of the confidentiality of our dataset, we cannot report the names of the institutions which limits the interest of the figure, so we omit it.

\chapter{Data interpretability}
\label{chap:interpretability}
In this chapter, we first argue why models can be used to interpret data. We then draw the landscape of model interpretability techniques with a special focus on Shapley values.
Stepping back from the current trends in model interpretation, we motivate our approach through algorithmic information theory and define the \textit{Razor entropy}. Finally, we present \textit{top-k Shapley values}, a novel value distribution method for coalitional games that presents special properties on games with a submodular characteristic function. We then use it to present a model-agnostic feature importance evaluation algorithm that runs in quadratic time in the number of features. 

\section{From model interpretability to data interpretability}

In today's world, analyzing large quantities of data has become crucial in various fields such as physics, economics, biology or marketing.

However, as underlined by~\cite{varian2016causal}, some patterns are not observed in the data, which makes counterfactual studies (e.g. \textit{if the price of meat rose, would people consume less?}) impossible. This is where the predictive capabilities of (supervised) machine learning models can help. Indeed, a machine learning algorithm will try, with more or less success, to generalize the training samples it saw. Numerous subsequent papers ~\cite{mullainathan2017machine,  zhao2019causal, athey2019machine} also advocated for the use of machine learning models to interpret data.

Unlike data, models can be queried on instances that never happened. In this sense, the model becomes the source of knowledge. One can question the validity of this approach, we address this point in this chapter through the lens of Solomonoff's theory of inductive inference.

\subsection{Model interpretability}

Interpretable Machine Learning~\cite{molnar2020interpretable, guegan2020note} is a field that received a growing interest from the community throughout the last years.

\paragraph{Interpretability} Interpretability was defined in~\cite{miller2019explanation} as \textit{the degree to which a human can understand the cause of a decision}. This definition does not lift the veil on the meaning of \textit{understand}. Furthermore, it invokes causality. Here, we suppose that the observed variable depends on the explanatory features given as input to the model, thus that the model learns a conditional distribution.

\subsubsection{Categorizations}

Several categorizations exist for model interpretability techniques~\cite{molnar2020interpretable}.

\paragraph{Intrinsic vs post-hoc} For example, intrinsic methods rely on using simple models that are already fairly interpretable, like linear regressions or small number of rules. Intrinsic methods will typically enforce a low model complexity through the use of regularization for linear models (e.g. using Lasso~\cite{tibshirani1996regression} for sparsity) or by limiting the depth of decision trees.
On the other hand, post-hoc methods do not modify how the model is trained, even if they may use the model structure.

\paragraph{Model-specific vs model-agnostic} Model-specific methods make use of the specific model that is being used whereas model-agnostic methods just use the fact that the model supports a training and an evaluation method. The former treat the model as a so-called \textit{white box} while the latter are often referred to as \textit{black box} methods.

\paragraph{Scope of explanation} The scope of interpretability is an important characteristic: it lies on a spectrum, ranging from the local level (the method explains a single prediction) to the global level (the method allows to understand the model as a whole).

\paragraph{Explanation format}
Finally, the different interpretability techniques are categorized by how they present their results. Note that these categories are not mutually exclusive. For example, a (kernel) SVM will provide support vectors that are both model-specific but also prototypes.

\subparagraph{Model internals}
The most flexible methods, that are obviously model-specific, report the model internals. Examples are the coefficients of generalized linear models, or visualization of convolutional neural network filters~\cite{zeiler2014visualizing}. Deep learning also popularized some architectures that compute some normalized multiplicative weights in a mechanism called \textit{attention}. The values of these coefficients can also be used as part of a model interpretability process~\cite{serrano2019attention}.

\subparagraph{Feature statistic}
Other methods propose some feature statistic or summary. They include methods that provide individual or group feature importance. Some of these methods offer visualizations of the statistics they compute.

\subparagraph{Prototypes and counterfactuals}
Finally, a whole class of methods output data points from the same space as the input. These methods are either prototype based (they explain the model with some typical examples) or counterfactual (they explain specific predictions with inputs that look similar but obtain different outputs).

\subsection{Shapley values for feature importance}

In this section, we describe a very popular approach to model interpretability.

\subsubsection{Shapley values}

Shapley values come from the field of economics, where they were introduced by Lloyd Shapley~\cite{shapley1953value}.

Let us consider a finite set $N$ of elements we call \textit{players} and a function $f: 2^N \rightarrow \R$ such that $f(\{\})=0$ that we call the value or characteristic function. Together, these objects define a coalitional game.

This game can be interpreted as follows: for any subset $S \subset N$, $f(S)$ is the value of this subset and models the payoff that the members of $S$ can expect when working together.

Economists are interested in allocating a value $\phi_i$, to each player, and make this value satisfy some axioms.

\paragraph{Efficiency} The sum of all attributions is equal to the value of the total coalition:
$$\sum_i \phi_i = f(N)$$
This axiom has a practical interest when redistributing physical payoffs as it balances the earnings and payouts.

\paragraph{Symmetry or Anonymity} The identity of players is not taken into account, that is, two identically performing actors will receive the same value.\\
If for all $S \subset N$ such that $i \notin S, j \notin S$, $f(S \cup \{i\}) = f(S \cup \{j\})$ then $\phi_i =  \phi_j$.

\paragraph{Null effects} A player who does not contribute gets value $0$. If for all $S \subset N$, $f(S \cup \{i\}) = f(S)$, then $\phi_i = 0$.

\paragraph{Linearity} This axiom is the most particular as it involves another game. Let us consider player values $\phi_i'$ according to another game on the same set of players $N$ but using a different value function $g$, as well as player values $\phi_i"$ defined for the \textit{sum} game with value function $f+g$. Then they must satisfy for every $i \in N$:

$$\phi_i" = \phi_i + \phi_i'$$

To these axioms, one can add the following:

\paragraph{Monotonicity} With the above notations, if for all $S \subset N$ such that $i \notin S$, $$g(S \cup \{i\}) - g(S) \geq f(S \cup \{i\}) - f(S)$$
then $$\phi_i' \geq \phi_i$$

Shapley showed that the four first axioms are sufficient to ensure the existence and uniqueness of the solution. In fact, one can replace linearity and null effects with the monotonicity axiom without changing the result~\cite{young1985monotonic}.

\begin{definition}
The Shapley value is the average across all ordered coalitions of the marginal contribution of a player to the set of players already in the coalition. It can be computed as:
$$\phi_i = \sum_{S \subset N, i \notin S} \frac{|S|!(|N|-|S|-1)!}{|N|!}(f(S \cup \{i\}) - f(S))$$
\end{definition}

\subsubsection{Shapley feature importance}

Shapley values to measure feature importance have been popularized by~\cite{lundberg2017unified}, although they were already applied before. They named their approach SHAP for SHapley Additive exPlanations.

The goal is to explain the algorithm prediction $\tilde y$ on a single input $x$. The players $N$ are the features and the characteristic function is the prediction for feature subsets. More precisely, it is the prediction of the model when replacing the missing features with zeros.

\paragraph{Issues} There are several issues with this procedure.

\subparagraph{Filling with zeros} Replacing inputs with zeros is not always possible or a good option. For categorical features, there is no such mean value. Even when inputs are centered, zero can have no meaning, e.g. for a balanced variable taking the values $\{-1, +1\}$. This issue can be mitigated by sampling the missing values over the rest of the dataset, as first suggested (for a global procedure over the whole dataset) in~\cite{breiman2001random}.

\subparagraph{Independence} However, even sampling missing values supposes that the features are somewhat independent. As an example, suppose that one wants to make predictions on individuals from their height and weight. The procedure above would measure the contribution of weight by comparing a prediction made from both variables with a prediction made with the height and a fake weight, randomly sampled. In general, height and weight are strongly correlated but the model will be asked to make predictions on combinations it has rarely seen in the training set, thus leading to absurd interpretations.

\subparagraph{Categorical dependent variables} Finally, the target can be categorical. Then, what value should be modelled by the characteristic functions? Should it be the probability of the true class, its logarithm or the last network activation before the softmax?

\subsection{Rash\=omon effect and solutions}

Finally, the largest issue occurring in the above methodology is the so-called \textit{Rash\=omon effect}~\cite{breiman2001random, fisher2019all}. This effect is named after the 1950 Kurosawa film where the same person is described in contradictory terms by four witnesses. Similarly, different models can fit the data equally well while working very differently. Thus,~\cite{mullainathan2017machine} fit 10 LASSO regressions on the same dataset and find strong disparities in the coefficients.

This effect motivates the concepts introduced in~\cite{fisher2019all}: model reliance, model class reliance and algorithm reliance.

\paragraph{Model reliance} Model reliance is the procedure causing the Rash\-omon effect: measuring how much a single model relies on a feature is unstable because it is possible to achieve the same performance in very different ways.

\paragraph{Model class reliance} \cite{fisher2019all} advocate for model class reliance. Model class reliance is the interval of the reliance values of individual models from a given class that perform above a baseline threshold. They provide theoretical guaranties to estimate model class reliance.

\paragraph{Algorithm reliance} Finally, they define as algorithm reliance the procedures that run the model-fitting algorithm multiple times, on different subsets of features. Hence, it does not measure the reliance of a single model but instead compares the performances of multiple models outputted by the same learning algorithm.

An example of algorithm reliance is the \textit{forward stepwise} method described in~\cite{gevrey2003review}. This procedure starts from an empty set of variables and recursively adds the variable maximizing the improvement of the error value until all variables have been selected. Their algorithm is similar to the \textit{top-k Shapley value} we propose. 

A similar algorithm was used in~\cite{cohen2005feature} as a feature selection technique. Interestingly, instead of considering the marginal improvement of the metric, they use a low-order approximation of the Shapley value to choose the variables they add.

\section{An (Algorithmic) Information Theory view on variable importance}
\label{sec:ait}

The choice of models in machine learning is quite unsettling from an epistemological point of view as it supposes some congruence between the world structure (or at least the structure of the patterns appearing in the data) and the models that have been described by the literature and implemented so far.

Fortunately, there are seemingly universal tools that free us from any model choice, at least in theory. We first introduce the reader to the Solomonoff theory of inductive inference and present some elementary results of algorithmic information theory that justify its relevance for model-agnostic interpretations of data. We then apply a connection with classical information theory to motivate the use of neural networks to approximate conditional entropy as a proxy for Kolmogorov complexity.

\subsection{Introduction to Algorithmic Information Theory}

The starting point of AIT is the Church--Turing thesis that supposes the existence and universality of a concept of computable functions. Using the existence and practical definition of a universal Turing machine being able to simulate any other, one is able to associate with any computable function a minimum program \textit{length} that will be well defined for a given description language and equivalent up to an additive constant across all possible implementations of a Turing machine. The Kolmogorov complexity of data $x$ is defined as the minimum program length $K(x)$ of the computable function that takes no input and outputs this data.

Thus, Solomonoff's theory of inductive inference explains observations by supposing they were generated by the program defining their Kolmogorov complexity, that is, the smallest program that can produce those observations.

It is also possible to define a conditional Kolmogorov complexity $K(y | x)$ that measures the minimum length of a program trying to output $y$ when given $x$ as input~\cite{li2008introduction}. From now on, $y$ is a categorical variable.

However, Kolmogorov complexity is not computable.

\subsection{From Algorithmic Information Theory to Machine Learning}

In this section, we approach machine learning from the point of view of Kolmogorov complexity.

Given training and test sets, one would intuitively define the \textit{best} prediction by computing the minimal program producing $y_{train}$ from $x_{train}$ and \textit{evaluate} it on $x_{test}$. However, nothing guarantees that the program will run correctly. Also, we did not define any evaluation metric yet that can measure the success of our operation.

A possible solution is to add some conditions on the output program $P$ by enforcing that not only $P(x_{train}) = y_{train}$, but also that $P(x_{test})$ is a valid computation. Hence we see that $P$ also depends on $x_{test}$.\\
To make a comparison with classical machine learning algorithms, if some categorical feature presents classes that are only present in the testing dataset, the prediction function will probably raise an error in most implementations, unless there is a mechanism to catch those exceptions. Such a mechanism would not exist in the shortest program.

Instead, the interesting quantity to be measured is $K(y_{test} | x_{train}, y_{train}, x_{test})$. This solves all our problems. First, we do not forget about the training data that can be used at will. We included $x_{test}$ in our data so the program will run correctly on it. Furthermore, trying to output $y_{test}$ means that if we cannot find a simple pattern in the training data, we will have to encode $y_{test}$ itself in our program code. Hence, the optimal program will be a compromise between the \textit{predictable} and \textit{unpredictable} components of $y_{test}$, the unpredictable components being the similar to an error.
We note the predictable component $M$ and the error component $E$.

$$K(y_{test} | x_{train}, y_{train}, x_{test}) \approx M + E$$

Here, we see a connection to model selection. Since $K(y_{test} | x_{train}, y_{train}, x_{test})$ is an absolute lower bound on any function that can output $y_{test}$ from $x_{train}$, $y_{train}$, $x_{test}$, it applies to machine learning models. Criterion like AIC or BIC try to minimize not only the model error but the sum of the error and an estimate of the model size with a similar decomposition in a model size and an error. Furthermore, cross-entropy loss encodes exactly the additional quantity of information needed to correctly decode $y_{test}$ from the model output. Note that in this case, the model can even be trained on $y_{test}$ but just does not have access to it at inference time.

Hence, we argue that summing the cross-entropy of a machine learning classifier and its size is an effective upper bound on the conditional Kolmogorov complexity. 

This result is not surprising, as compression programs have been used for years as proxies to measure Kolmogorov complexity~\cite{cilibrasi2005clustering} and there are now programs applying deep learning to lossless compression~\cite{bellard2019lossless}.

Note that the same reasoning would not hold for regression tasks: neural networks do not consider floating-point representations as bit strings but as real values. Thus, the Hamming distance between the binary representation of $4$ and $4.321$ is 8 whereas the distance between $4$ and $32$ is 2. Values that would yield a better approximation in the space of real numbers often perform worse when looking at them from the point of view of information theory.

\subsection{Forgetting model size: the case for out-of-sample estimates of conditional entropy}

However, deep learning is most successful with overparametrized models~\cite{du2018gradient}, which compromises the use of regularization on model sizes.
Here we argue that simply measuring the generalization loss of classifiers gives a meaningful quantity.

Supposing that we can vary the length of the testing set, because we have virtually infinitely many samples, we now suppose that the test set is infinite and consider the $n$ first items: $\left(x_{test:n}, y_{test:n}\right)$

From the above intuition, $M$ should not augment with $n$ when $n$ becomes large enough because it is a pattern common to all of the data, whereas $E(n)$ depends on $n$. Furthermore, $E(n)$ being unpredictable, it should increase linearly. But because Kolmogorov complexity is bounded by linear functions of the length, it cannot increase more than linearly. Thus, we write:

$$K(y_{test:n} | x_{train}, y_{train}, x_{test:n}) \approx M + n \cdot E$$

Hence, the quantity

$$\lim_{n \rightarrow \infty} \frac1n K(y_{test:n} | x_{train}, y_{train}, x_{test:n})$$

is well defined and will converge either to $0$ or to some positive quantity that we identify with the error $E$.

A classical result in information theory now tells us~\cite{kaitchenko2004algorithms} that 

$$\lim_{n \rightarrow \infty} \frac1n K(y_{test:n} | x_{train}, y_{train}, x_{test:n}) = H(y_{test} | x_{train}, y_{train}, x_{test}) =: E$$

where $\left(x_{test}, y_{test}\right)$ are considered as finite-order stationary Markov sources and $H$ is the (conditional) entropy.

Another classical inequality in AIT is:

$$K(A|B) \lesssim K(A) \lesssim K(A|B) + K(B)$$

Hence, we can eliminate the finite-sized $x_{train}$ and $y_{train}$ in the Kolmogorov complexity limit. Similarly, by conditional independence, $y_{test}$ does not depend on the training set and we can simplify the above to:

$$\lim_{n \rightarrow \infty} \frac1n K(y_{test:n} | x_{test:n}) = H(y_{test} | x_{test}) =: E$$

However, Kolmogorov complexity is not computable in general. We want to approximate $K(y_{test} | x_{train}, y_{train}, x_{test})$ with a simpler value.
From the previous computations:

$$K(y_{test:n} | x_{train}, y_{train}, x_{test:n}) \approx n \cdot E = n \cdot H(y_{test}, x_{test})$$

Conditional entropy is frequently mentioned in machine learning as the lower bound of the cross-entropy $H(p, c)$ of a classifier $c(y|x)$ trying to learn the true distribution $p(y|x)$:

\begin{align*}
    H(y|x)
    &= -\mathbb{E}_{p(y|x)}[\log p(y|x)]\\
    &\leq -\mathbb{E}_{p(y|x)}[\log c(y|x)]\\
    &= H(p(y|x), c(y|x))
\end{align*}

Hence, the loss of a classifier is an upper bound for $E$, the marginal size of a minimal program reconstructing $y$ from $x$.

From the previous considerations on Kolmogorov complexity, it makes sense to use a training set to better approximate $E$ when being given a finite quantity of testing data $n$, as the expression we want to approximate also has access to $x_{train}, y_{train}$. Therefore, when measuring the test cross-entropy

$$H(p(y_{test} | x_{test}), c(y_{test} | x_{train}, y_{train}, x_{test}))$$

one assumes that the classifier is the result of a training process $(x_{train}, y_{train})$ that produced weights $\Theta(x_{train}, y_{train})$. Hence, we write 

\begin{align*}
    &\frac1n K(y_{test:n} | x_{train}, y_{train}, x_{test:n})\\
    &\approx \frac1n K(y_{test:n} | x_{test:n})\\
    &\rightarrow H(y_{test} | x_{test})\\
    &\leq H(p(y_{test} | x_{test}), c_{\Theta(x_{train}, y_{train})}(y_{test} | x_{test}))
\end{align*}

The above inequality has several causes:
\begin{enumerate}
    \item it is in general impossible to compute the actual value of Kolmogorov complexity, let alone the shortest program, hence we can only produce rigorous upper bounds or sketchy approximations like we did
    \item $c$ does not learn from $y_{test}$ (some semi supervised models can learn from $x_{test}$ though)
    \item the model class $c$ might be too limited and the learning process does not give an optimal solution $\Theta$
\end{enumerate}

On the other hand, to compensate the both practical and theoretical limitations of the optimization process, we close the gap by cheating and using an overparametrized model $c$, far from being the shortest program. This creates a problem as the $M$ component can grow indefinitely to learn $y_{test}$ while $E$ is underestimated (this is called overfitting).\\
Fortunately, it is possible to ignore any finite (reasonable) model size $M$ by supposing $n \rightarrow \infty$. Therefore, $\Theta$ is produced mainly from the training data, so that our trainable surrogate $c$ does not try to learn $y_{test}$ and is more likely to generalize to an infinite testing. Note that in our experiments, we allow the model to look at the test data for hyperparameter search.

Hence, we showed how classifiers can provide a reasonable proxy for Kolmogorov complexity through out-of-sample estimates of cross-entropy loss.

\section{Razor entropy}
\label{sec:razor}

Our transaction graph objective can be best understood as an instance of a more general quantity that we define below.

\begin{definition}
    Given two finite sequences $a_i$ and $b_i$ of length $n$, we define the Razor Kolmogorov complexity of $a$ and $b$ as:
    $$RK(a \square b) := min_{z \in \{0,1\}^n} K\left( (\{a_i,b_i\}[z_i])_{1 \leq i \leq n} \right)$$
    where $\{a_i,b_i\}[z_i]$ means $a_i$ if $z_i = 0$ else $b_i$.

    Given two infinite sequences $a_i$ and $b_i$, we define the Razor entropy of $a$ and $b$ as:
    $$RH(a \square b) := \lim_{n \rightarrow \infty} \frac1n RK(a_{:n} \square b_{:n})$$
\end{definition}

An intuitive way to understand this is to imagine that one tries to output any sequence of elements of either $a$ or $b$.

The following theorem motivates the definitions of section \ref{sec:objective}.

\begin{theorem}
    \label{thm:razor}
    When $a$ and $b$ are two sequences of discrete and iid random variables over a finite set $E$ such that $\Pr[\{a,b\} = \{i,j\}] = p_{ij}$ (they are un-ordered pairs),  the Razor entropy can be computed as:
    $$RH(a \square b) := \min_q - \sum_{i,j} p_{ij} \log \max(q_i, q_j)$$
    where $q$ is a probability distribution.
\end{theorem}
\begin{proof}
    We suppose that there is a deterministic function $z: E \times E \rightarrow E$ such that $z(i, j) = z(j, i) \in \{i,j\}$ and such that $RK(a_{:n} \square b_{:n}) = K\left((z(a_i, b_i))_i\right)$. In other words, we admit that the optimal encoding is deterministic.\\
    We define the distribution of the target $z(a, b)$:
    $$q(z)_i = \Pr[z(a,b) = i] = \sum_{j | z(i,j) = i} p_i$$
    Thus, we deduce that the Razor entropy of $a$ and $b$ is the entropy of their optimal encoding:
    $$RH(a \square b) = \min_z H(q(z)) = \min_z - \sum_i q(z)_i \log q(z)_i$$
    Thanks to the properties of cross-entropy, we know that:
    $$H(q(z)) = \min_q - \sum_i q(z)_i \log q_i$$
    where $q$ is any distribution.
    We develop $q(z)_i$:
    $$H(q(z)) = \min_q - \sum_{i,j} p_{ij} \log q_{z(i,j)}$$
    We can plug this in the first expression, and remember that $z(i, j) \in \{i,j\}$:
    \begin{align*}
        RH(a \square b)
        &= \min_z H(q(z))\\
        &= \min_z \min_q - \sum_{i,j} p_{ij} \log q_{z(i,j)}\\
        &= \min_q \min_z - \sum_{i,j} p_{ij} \log q_{z(i,j)}\\
        &= \min_q - \sum_{i,j} p_{ij} \log \max(q_i, q_j)\\
    \end{align*}
\end{proof}

Furthermore, one can make use of additional data and replace the probability, entropy and complexity with their conditional counterpart without changing the proof. It is also possible to make the pairs $\{a,b\}$ ordered, and even use different sets for the values of $a$ and $b$ since our definitions are oblivious to the chosen encoding.

Two particular cases that are not covered by our theory are the conditioning on $z_i$, the Boolean choice that makes sense only for ordered pairs, and on $\{a_i,b_i\}[\neg z_i]$, the variable that is not predicted. We conjecture that Theorem \ref{thm:razor} still applies in this case and leave the verification for future works.

\section{A novel cost-sharing rule}

In this section, we present a cost-sharing rule that has similar properties to the Shapley value but is easier to compute in games that have certain properties.

\subsection{Top-k efficiency}

From the 4 axioms characterizing Shapley values uniquely (efficiency, symmetry, null effects and linearity), we remove the linearity. Indeed, when we only consider one game, linearity does not apply. Instead, we consider a new rule.

\paragraph{Top-k efficiency} For any coalition size $k$, the maximal value of a coalition must be equal to the sum of the $k$ largest contributions. In other terms, for every $0 \leq k \leq |N|$
$$\max_{S\subset N, |S| = k} f(S) = \sum_{i=|N|-k+1}^{|N|} \phi_{(i)}$$

where $\phi_{(1)} \leq \phi_{(2)} \leq \cdots \leq \phi_{(|N|)}$ is the order statistic of $\phi_i$. Such a set $S$ is called \textit{top-k efficient}.

\begin{remark}
    Top-k efficiency implies efficiency (take $k = |N|$).
\end{remark}

\begin{remark}
    A cost-sharing rule verifies top-k efficiency if and only if there is a permutation of the players such that for every $0 \leq k \leq |N|$, $\{1,2, \cdots, k\}$ has the highest value among all coalitions of size $k$.
\end{remark}

Thus, this axiom appears to make strong suppositions on the game. We will see that a relaxed version of top-k efficiency can be satisfied by a wide and natural class of games.

\subsection{Submodular games}

We are particularly interested in games with a submodular characteristic function.

\begin{definition}
    A submodular set function is a function $f: 2^N \rightarrow \R$ such that for every $X \subset N$, and every $x_1, x_2 \notin X$, we have
    $$f(X \cup \{x_1\}) + f(X \cup \{x_2\}) \geq f(X \cup \{x_1, x_2\}) + f(X)$$
\end{definition}

Submodular set functions occur in various real-world problems are they formalize the concept of diminishing returns. A common example is the coverage problem. They also possess an interesting property:

\begin{proposition}
    \label{prop:submodularity}
    Maximizing a monotone submodular function subject to a cardinality constraint can be done with a $1-\frac1e$ approximation factor with a greedy algorithm.
\end{proposition}
\begin{proof}
    See~\cite{nemhauser1978analysis}.
\end{proof}

\begin{remark}
    The greedy algorithm is oblivious to the cardinality constraint.
\end{remark}

\subsection{Top-k Shapley values}

The consequence of Proposition \ref{prop:submodularity} is that submodular games can verify a relaxed version of top-k efficiency.

\paragraph{Relaxed top-k efficiency} For every $0 \leq k \leq |N|$

$$\left(1-\frac1e\right) \max_{S\subset N, |S| = k} f(S)  \leq \sum_{i=|N|-k+1}^{|N|} \phi_{(i)}$$

\begin{definition}
    \label{def:topk}
    Top-k Shapley values. The top-k Shapley values $\phi_i$ are defined recursively along with an ordering $\sigma$ in the following way:
    $$
    \left\{
    \begin{array}{ll}
    \sigma(N) =& \argmax_i f(\{i\})\\
    \phi_{\sigma(N)} =& f(\{\sigma(N)\})\\
    \sigma(k) =& \argmax_i f(\{\sigma(N), \sigma(N-1), \cdots, \sigma(k+1)\} \cup \{i\}) \\
    &- f(\{\sigma(N), \sigma(N-1), \cdots, \sigma(k+1)\})\\
    \phi_{\sigma(k)} =& f(\{\sigma(N), \sigma(N-1), \cdots, \sigma(k)\})\\
    &- f(\{\sigma(N), \sigma(N-1), \cdots, \sigma(k+1)\})\\
    \end{array}
    \right.
    $$
    Ties are resolved arbitrarily.
\end{definition}

\begin{theorem}
    In a game with a monotone submodular characteristic function, top-k Shapley values are a cost-sharing rule satisfying efficiency, symmetry, null effects and relaxed top-k efficiency.\\
\end{theorem}
\begin{proof}
    This follows directly from Proposition \ref{prop:submodularity}.
\end{proof}

\begin{theorem}
    Top-k Shapley values can be computed in $\mathcal{O}(|N|^2)$ evaluations of $f$.
\end{theorem}
\begin{proof}
    Definition \ref{def:topk} defines a greedy algorithm. $|N|$ values have to be computed and each value is a max over $\mathcal{O}(|N|)$ terms, each of which requires a constant number of function evaluations. Hence the total number of function evaluations is $\mathcal{O}(|N|^2)$.
\end{proof}

\section{Applying top-k Shapley values to data interpretability}

In this section, we apply the previously defined top-k Shapley values to the unsupervised objective defined in section \ref{sec:objective}.

As a reminder, we model edges $(a, b)$ with a context $c$ to maximize the log likelihood:
$$\max_{\theta,direction} \frac1n \sum_i \log f_\theta(y_i, c_i)[x_i]$$
where $x_i, y_i = b_i, a_i~\text{if}~direction_i ~\text{else}~a_i, b_i$.

Hence, we define the set of players $N$ as the features contained in $c$. For any $S \subset N$, we note $c_{|S}$ the restriction of $c$ to the features contained in $S$ (and $c_{i|S}$ the value of $c_{|S}$ for the $i$-th item) and use the following empirical value function $f$ (not to be confused with the model $f_\Theta$):
$$f(S) := - \max_{\theta,direction} \frac1n \sum_i \log f_\theta(y_i, c_{i|S})[x_i] + C$$

Note the minus sign, as $f$ the likelihood is a decreasing function. $C$ is a normalization constant to ensure $f(\{\}) = 0$.

As justified in the previous section, the average is made over the testing set while $\Theta$ is the result of a learning algorithm on the training set.

\begin{remark}
    Our unsupervised objective is not a cross-entropy of actual variables. However, it can be understood as approximating the Kolmogorov complexity of $y$ given $x$ where $x$ and $y$ represent $a$ and $b$ swapped to make the result as small as possible.
\end{remark}

\begin{lemma}
    Conditional entropy gains are \textbf{not} submodular, as stated in remark 13 of~\cite{krause2008near}.
\end{lemma}
\begin{proof}
    We give a simpler proof for discrete random variables here: let us consider 2 random bits $b_1, b_2$ and $b_3 = b_1 \oplus b_2$. $H(b_3) = H(b_3 | b_2) = H(b_3 | b_1) = 1$ and $H(b_3 | b_1, b_2) = 0$. Hence, $H(b_3 | b_1, b_2) - H(b_3 | b_1) < H(b_3 | b_2) - H(b_3)$, that is $b_2$ helps more when combined with $b_1$ than alone.
\end{proof}

Hence, the theoretical guarantees of top-k Shapley values generally do not hold in this setting. However, the absence of provable submodularity should not be viewed as an obstacle to the use of top-k Shapley values, for two reasons:
\begin{enumerate}
    \item Human decision makers sometimes unconsciously apply strategies that work best when the rewards follow a law of diminishing returns (submodularity)~\cite{kubanek2017optimal}. This shows that submodularity is consistent with human intuition and should help in simplifying interpretation.
    \item Even when the reward function is not submodular but is either approximated by a submodular function~\cite{narasimhan2012pac} or satisfies a weakened form of submodularity~\cite{das2011submodular}, greedy algorithms still achieve constant factor approximation. This approach is taken by~\cite{khanna2017scalable} that performs greedy feature selection by optimizing the $R^2$ statistic, which satisfies a weakened form of submodularity.
\end{enumerate}

We define $h(S)$ as the test loss of a model taking features from $S$ as input. Since $h$ is a decreasing function (in theory) and we want to ensure $f(\{\}) = 0$, we define $f(S) := h(\{\}) - h(S)$. Furthermore, as our estimation is very noisy because of the small dataset size, we average the test values over 5 runs.

\begin{center}
    \begin{tabular}{|c c c|} 
    \hline
    Order & Feature   &$\phi$ \\
    \hline\hline
    1&Entity & 0.858\\
    \hline
    2&Product & 0.143\\
    \hline
    2&Dealer spread & 0.011\\
    \hline
    4&Price & 0.010\\
    \hline
    5&Notional & 0.023\\
    \hline
    6&Time & 0.000\\
    \hline
    7&Direction & 0.004\\
    \hline
    8&Day & -0.003\\
    \hline
    9&Market price & -0.028\\
    \hline
\end{tabular}
\end{center}

We observe that $\phi$ globally decreases, although some error exists as we use experimental values.

It is also possible to draw all marginal improvements that were computed, i.e. the values of $f(S \cup \{i\}) - f(S)$ for all $S$ that are top-k efficient and $i \notin S$. We report the results in Figure \ref{fig:delta}.

Obviously, the choosing entity is the most important feature, as different dealers trade differently. The product is second, as dealers specialize in different products. 

Dealer spreads come third, which is coherent with the findings of \cite{du2019counterparty}. Dealer preferences change over time and spreads are correlated with how close a sample is to our testing set, hence one could argue that dealer spreads are indirectly correlated with dealer choice, being a mere proxy for time. However, one can notice that price comes just after, with a very similar top-k Shapley value, even though it contains less temporal information. This suggests that both dealer spreads and prices are directly useful to the prediction and not an artifact of our temporal splitting.

It is particularly interesting to note that the market price is not taken into account at this point and has a very low importance. It means that the price information merely reflects market trends and not a difference with a "fair" price. The marginal contribution of notional grows after learning the price, which is hard to explain other than by spurious correlations.  Finally, the top-k Shapley values of subsequent features are much less and much closer to being noise.

\begin{figure}
    \centering
    \includegraphics[width=\textwidth]{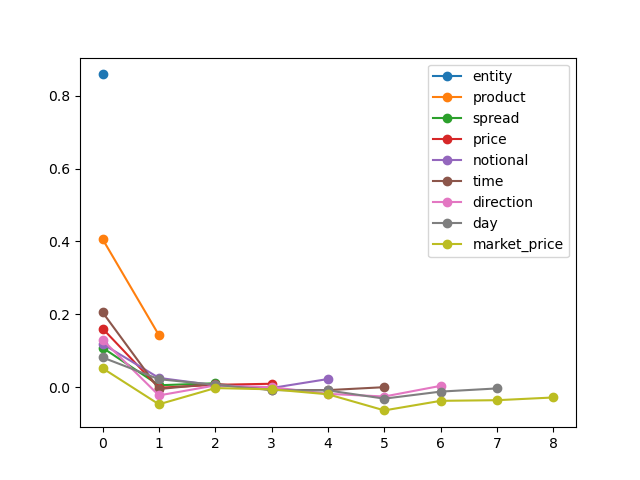}
    \caption{Marginal gains of test loss for each feature. Each feature is associated a different curve displaying its improvement over one of the top-k efficient subsets for all values of $k$.}
    \label{fig:delta}
\end{figure}
\chapter{Improvements to \textit{node2vec}}

A possible extension of our model would process a recent history of transactions, including those done with non-dealer entities, for example using recurrent neural networks instead of multilayer perceptrons.  This could allow to better explore intermediation on the OTC market \cite{li2019dealer}. 

However, non-dealers trade infrequently, which can hinder the learning process. We propose to use \textit{node2vec} embeddings as a feature.
The graph would have tens of thousands of nodes, which could be challenging if it was not sparse.

In this chapter, we describe improvements we made to the \textit{node2vec} algorithm \cite{grover2016node2vec} to make it more scalable for large graphs.
First, we explain the original node2vec algorithm. Then, we present our improved neighbor sampling algorithm. Finally, we describe our implementation.

\section{Node2vec}

\textit{Node2vec} is a common technique used to produce low-dimensional representations of nodes in a directed or undirected graph. The idea of node2vec is that a node should be embedded similarly to its \textit{neighborhood}. The neighborhood of a node is defined by using random walks. Given these random walks, one can apply the \textit{word2vec} algorithm \cite{mikolov2013efficient} to compute representations that try to maximize a negative sampling version of Skip-gram.

\subsubsection{Skip-gram}

Given a sequence of words $w_1, w_2, \cdots, w_T$, and a window size $c$ that defines the context, the objective is to maximize $$\sum_s \sum_{t=s-c}^{s+c}\log  p(w_t | w_s) $$

\subsubsection{Negative sampling}

An obvious parametric model for $p(w_t | w_s)$ would be $$p(w_t | w_s) \sim \exp(v_{w_s} \cdot v_{w_t})$$
However, the computation of the normalization constant is not usually tractable. Therefore, the alternative chosen by word2vec is to train a binary classifier distinguishing $w_s$ from words coming from a \textit{noise} distribution $\mathcal{D}$. The noise distribution is usually the unigram distribution raised to the power $0.75$. The formula is the following:

$$\log p(w_t | w_s) \approx \log \sigma(v_{w_s} \cdot v_{w_t}) - \mathbb{E}_{w\sim\mathcal{D}} [\log \sigma(-v_{w_s} \cdot v_{w})]$$

The expectation is approximated with a constant number of samples, e.g. 5.

\section{Improved neighbor sampling}

\subsection{Biased walks sampling}

To formulate node embeddings in a graph as word embeddings in sentences, one uses random walks. Node2vec uses second-order Markov chains.

Starting from a node, one produces a random walk by repeatedly sampling a neighbor of the last visited node. The walk is biased using two parameters $p$ and $q$.

Supposing we just traversed the edge $(t, v)$, and are taking a decision in $v$, we assign a transition probability to $(v, x)$ that is \textit{proportional} to $weight(v, x) \cdot \alpha(t, x)$. $weight(v, x)$ is the weight of this edge and 

$$
\alpha(t, x) = 
\begin{cases}
    \frac1p & \text{if } t = x \\
    1 & \text{if } dist(x, t) = 1 \\
    \frac1q & \text{if } dist(x, t) = 2
\end{cases}
$$

\begin{figure}
    \centering
    \includegraphics[width=.5\textwidth]{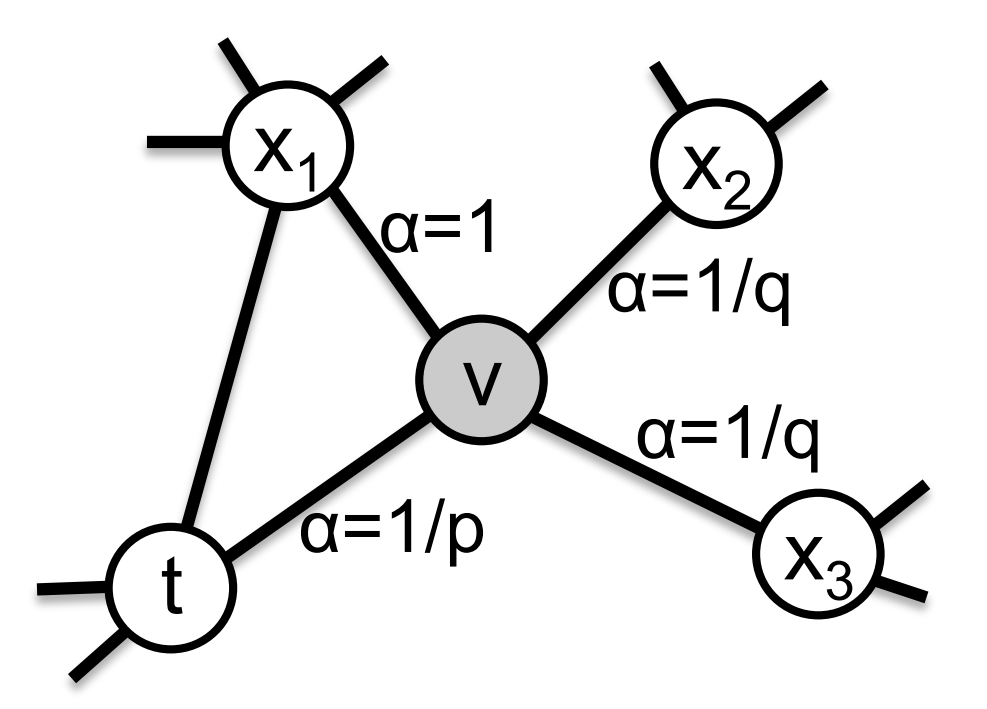}
    \caption{We just traversed the edge $(t, v)$, and are taking a decision in $v$. Depending on the values of $p$ and $q$, one is more likely to go back or forward.}
    \label{fig:walk}
\end{figure}

The problem of the above approach is that the transition probabilities depend on the two last visited nodes. The authors suggest to store them for every such pair first, which approximately multiplies the space complexity by the average degree $\bar d$. In that case, it is possible to  sample in $\mathcal{O}(1)$ using the alias method \cite{walker1977efficient}.

Without additional space, it is still possible to compute the probabilities on the fly in $\mathcal{O}(\bar d)$ by merging sorted adjacency lists. We remark that in most use cases, the model is trained for a small number of epochs, and thus this sampling method should be preferred.

\subsection{Rejection sampling}

Instead, we propose a method based on rejection sampling.  Rejection sampling can be stated formally by the following lemma.

\begin{lemma}
    Given a discrete probability distribution $\Pr[X = i] = p_i$, a sampling function $f$ for this distribution, the ability to sample uniform random variables in constant time and weights $0 \le w_i \le 1$, one can sample from a distribution $\Pr[X = i] \sim p_i \cdot w_i$ in $\mathcal{O}\left( \frac{1}{\mathbb{E}_{i \sim p_i}[1-w_i]} \right)$ expected calls to $f$.
\end{lemma}

\begin{proof}
    The following procedure works:
    \begin{itemize}
        \item draw a sample $i$ using $f$
        \item draw a uniform random variable $r \in [0, 1]$
        \item if $r < w_i$, return $i$
        \item else, repeat
    \end{itemize}
    The probability of returning $i$ is proportional to $p_i \cdot w_i$, hence the algorithm is correct.
    The probability of failure is $\mathbb{E}_{i \sim p_i}[1-w_i]$, the expected success time is its inverse.
    Figure \ref{fig:rejection} summarizes the situation in a graphical way.
\end{proof}

\begin{figure}
    \centering
    \includegraphics[width=0.5\textwidth]{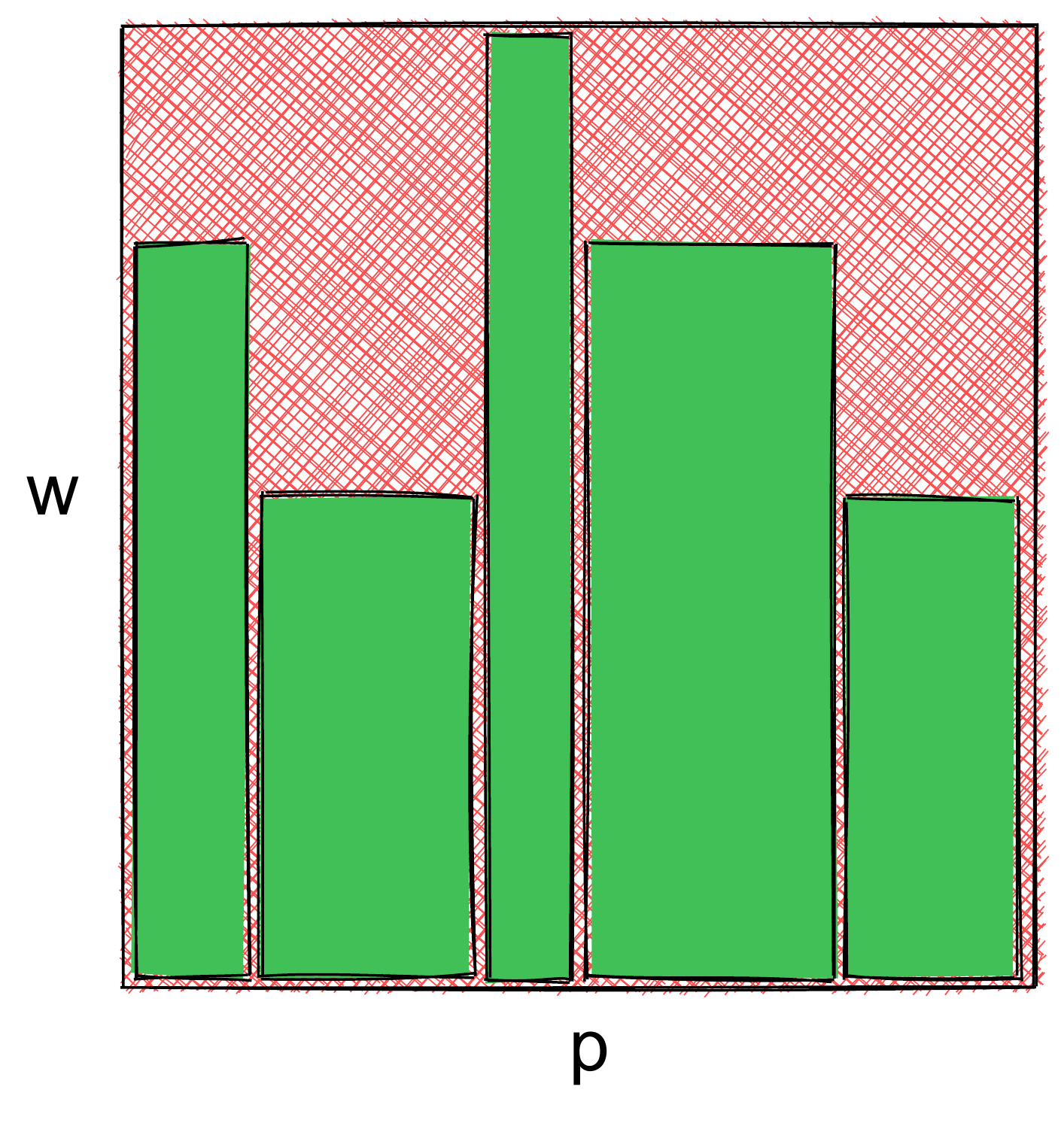}
    \caption{The algorithm first chooses a green bar $i$ according to the discrete distribution $p$, then compares its height $w_i$ with $r$ and returns $i$ or fails accordingly. The green area is the success area, the red area represents failure.}
    \label{fig:rejection}
\end{figure}

\subsection{Application}

We want to compute the next neighbor after visiting edge $(t,v)$. To apply rejection sampling, we take $p_x \sim weight(x, v)$ where $p_i$ was defined in the previous section and $weight$ denotes possible graph weights that may be oriented. In an unbiased graph, $weight(v, x) = 1$ for all neighbors $x$, and $p_x = \frac{1}{deg(v)}$ for normalization. We also choose $w_x = \alpha(t, x)$. Thus, $p_x \cdot w_x = weight(v, x) \cdot \alpha(t, x)$.

This procedure uses a linear time and space preprocessing that computes the normalized first-order transition probabilities from each node.

To evaluate the time complexity, we consider that sampling a neighbor can be done in constant time (using the alias method). Computing $w_x = \alpha(t, x)$ is a bit more costly, as one has to test whether $x$ is a neighbor of $t$. This can be done in $\mathcal{O}(\log deg(t))$ using binary search on sorted adjacency lists.

Finally, rejection happens with a probability at most $\beta = \max \left\{a/b | a,b \in \{1, p, q\}\right\}$. For typical values of $p$ and $q$ reported in the paper, $\beta < 4$ and can be considered constant.

Thus, we apply a $\mathcal{O}(n \log n)$ preprocessing and can sample in expected  $\mathcal{O}(\log D)$ time with $D$ the maximum degree.

\section{Implementation}

We implemented our modified node2vec sampling in Python.

\subsubsection{JIT compilation}
Generating the walks is time consuming as it uses loops that are slow in Python. We used \texttt{numba} \cite{lam2015numba}, a jit compiler. 
We store the graph using its adjacency matrix in compressed sparse row format.

We achieve performances such that the walk generation is unnoticeable when used to feed the word2vec training.

\subsubsection{Word2vec}

We used the \textit{word2vec} implementation provided by the \texttt{gensim} package ~\cite{vrehuuvrek2011gensim}. This implementation is widely used and tested and supports parallelization.

We exploited inheritance to produce a class that calls our random walks generation function from the parallel sections of the gensim code. Thus, when using parallelization, each thread will generate its own random walks. This avoids having to allocate processes separately for the walk generation and word2vec training.

\subsubsection{Difference with the  original implementation}

Because of the limitations of their pure Python implementation of the walk generation, the authors of the paper suggest to pre-compute and store the random walks. This uses considerable storage: when computing $20$ walks per node, each of length $100$ (values taken from their paper), the user effectively stores $2000$ times each node. Furthermore, as they noted, their algorithm is biased by the choice of walks.

Unlike the original implementation, we do not use pre-computed walks. Instead, each iteration of the word2vec training will generate one new walk per starting node. This choice improves both the storage use and the convergence of the algorithm.
\chapter{Conclusion}

In this work, we aimed to study the dynamics of counterpart choice on the over-the-counter market. We cleaned and used a dataset of credit default swap transactions containing information about central clearing.

We first used this dataset to model the initial margins on portfolios consisting of multiple products. Our analysis showed that the data quality is insufficient to infer the dynamics of initial margin requirements by central counterparts.

Then, we questioned the choice of counterparts on the interdealer market.
We derived a semi-supervised likelihood function for transaction data that can effectively be used to train end-to-end differentiable models. We trained deep learning models on our dataset, providing experimental evidence that our objective captures meaningful information and produces interpretable embeddings of categorical edge and node features.

On the theoretical side, we tackled the epistemological question of data interpretability and motivated the application of conditional cross-entropy to study real world data through the lens of algorithmic information theory, as well as the use of neural networks to measure it.

Borrowing ideas from algorithmic information theory, we defined the \textit{Razor entropy}, a principled measure of information for pairs of variable. We proved that its formula matches the likelihood function we used on our transaction dataset.

Finally, we applied game theory to create a new method of payoff distribution that provides unique properties for games with monotone submodular characteristic functions. Our \textit{top-k Shapley values} can be used to model the algorithm reliance on each explanatory variable. On our dataset, we concluded that dealer choices are mainly motivated by prior structural affinities. We also confirmed the hypothesis that dealer spread is a factor of differentiation, albeit the size of our dataset does not allow for precise estimations.

\subsubsection{Future work}

There are several ways in which our work could be extended.

Our models can be used to study intermediation on the OTC market, by evaluating recurrent models on a transaction history. As a byproduct of our work, we produced a more efficient \textit{node2vec} algorithm and implementation.

Communication and information sharing between dealers could also be researched by using graph neural networks to pass messages between nodes and observe the algorithm reliance on those connections.

On the theoretical part, we conjecture that theorem \ref{thm:razor} can be extended as explained.

We showed that the \textit{Razor entropy} can be used to learn node and edge representations on graph and temporal networks. Further research would be needed to compare its performance to other methods.

Finally, our protocol applying \textit{top-k Shapley value} to measure algorithm reliance is suitable for numerous use-cases of data interpretation.

\appendix
\appendixpage

\chapter{Model training code}
\label{code:train}

\begin{lstlisting}[language=Python]
def train(
    processor,
    model,
    loader_train,
    loader_test,
    restrict_choices,
    batch_size,
    patience,
    improvement,
    max_epochs,
    lr,
    progress_bar=True,
):
    optimizer = torch.optim.Adam(
        chain(processor.parameters(), model.parameters()), lr=lr
    )
    loss_fn = nn.CrossEntropyLoss(reduction="none")

    def eval_direction(data, direction):
        X, y, z = processor(direction, *data)
        preds = model(X)
        if restrict_choices:
            preds[torch.arange(len(X)), z] = float("-inf")
        return loss_fn(preds, y)

    def eval_loss(data):
        return torch.min(eval_direction(data, False), eval_direction(data, True)).mean()

    history = []
    best_loss = float("inf")
    epochs_since_best = 0

    with tqdm(leave=False, disable=not progress_bar) as pbar:
        for epoch in range(max_epochs):

            model.train()

            for data in loader_train.iter(batch_size, shuffle=True):
                optimizer.zero_grad()
                loss = eval_loss(data)
                loss.backward()
                optimizer.step()

            model.eval()

            with torch.no_grad():

                for data in loader_train.iter(0):
                    loss_train = eval_loss(data)

                for data in loader_test.iter(0):
                    loss_test = eval_loss(data)

            history.append(
                dict(loss_train=loss_train.item(), loss_test=loss_test.item(),)
            )

            loss = loss_train
            epochs_since_best += 1
            if loss < improvement * best_loss:
                best_loss = loss
                epochs_since_best = 0
            if epochs_since_best == patience:
                break

            pbar.update(1)
            pbar.set_description("%.04f %.04f" % (loss_train.item(), loss_test.item()))

    return history

\end{lstlisting}
\chapter{Hyperparameter search code}
\label{code:hyperparam}
\begin{lstlisting}[language=Python]
import optuna


def objective(trial):
    processor = Processor(
        products=products,
        dealers=dealers,
        entity_direction=trial.suggest_categorical(
            "entity_direction", [False, True]
        ),
        entity_dim=trial.suggest_int(
            "entity_dim", 1, len(entities)
        ),
        product_dim=trial.suggest_int(
            "product_dim", 1, len(products)
        ),
        day_dim=trial.suggest_int(
            "day_dim", 1, 3
        ),
    ).to(DEVICE)

    nb_layers = trial.suggest_categorical(
        "nb_layers", list(range(3))
    )
    layers = [
        trial.suggest_int(
            "layer_%i/%i" % (i + 1, nb_layers),
            32,
            256 if nb_layers == 1 else 128,
            log=True,
        )
        for i in range(nb_layers)
    ]
    if nb_layers > 0:
        dropout = trial.suggest_uniform(
            "dropout", 0, 0.8
        )
    else:
        dropout = 0

    model = MLP(
        processor.output_dim,
        *layers,
        processor.n_cp,
        dropout=dropout,
    ).to(DEVICE)

    use_batch = trial.suggest_categorical(
        "use_batch", [False, True]
    )
    if use_batch:
        log_batch_size = trial.suggest_int(
            "log_batch_size", 11, 16
        )
        batch_size = 2 ** log_batch_size
    else:
        batch_size = 0

    history = train(
        processor,
        model,
        loader_train,
        loader_test,
        restrict_choices=True,
        lr=trial.suggest_loguniform(
            "lr", 1e-4, 1e-2
        ),
        batch_size=batch_size,
        patience=50,
        improvement=0.99,
        max_epochs=10_000,
        progress_bar=False,
    )

    return history[-1]["loss_test"]


study = optuna.create_study()
study.optimize(objective, n_trials=200)
\end{lstlisting}
\chapter{Hyperparameter search results}
\label{chap:hyper-results}

\begin{figure}
    \centering
    \includegraphics[width=\textwidth]{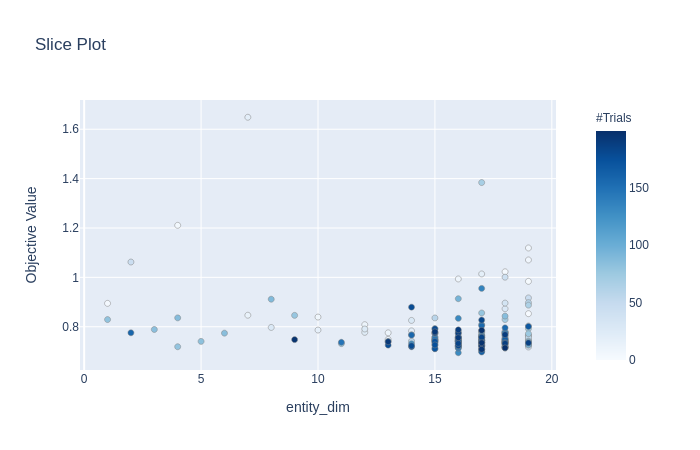}
    \caption{Entity dimension hyperparameter search. We chose $16$.}
\end{figure}

\begin{figure}
    \centering
    \includegraphics[width=\textwidth]{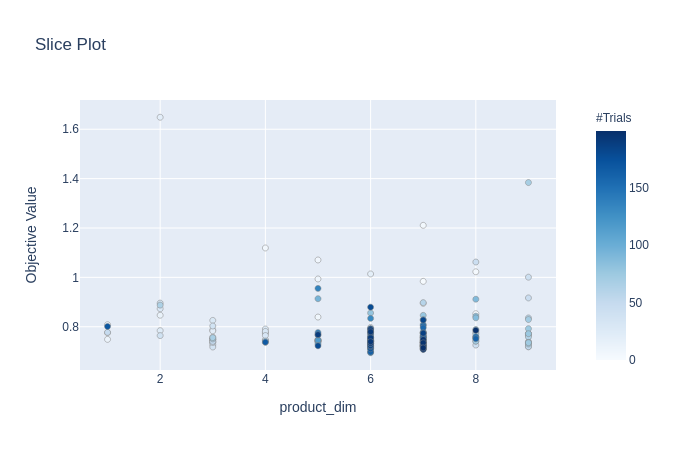}
    \caption{Product dimension hyperparameter search. We chose $6$.}
\end{figure}

\begin{figure}
    \centering
    \includegraphics[width=\textwidth]{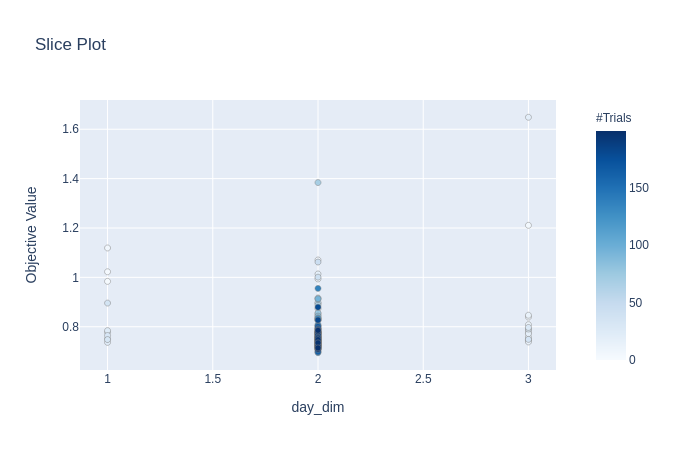}
    \caption{Day dimension hyperparameter search. We chose $2$.}
\end{figure}

\begin{figure}
    \centering
    \includegraphics[width=\textwidth]{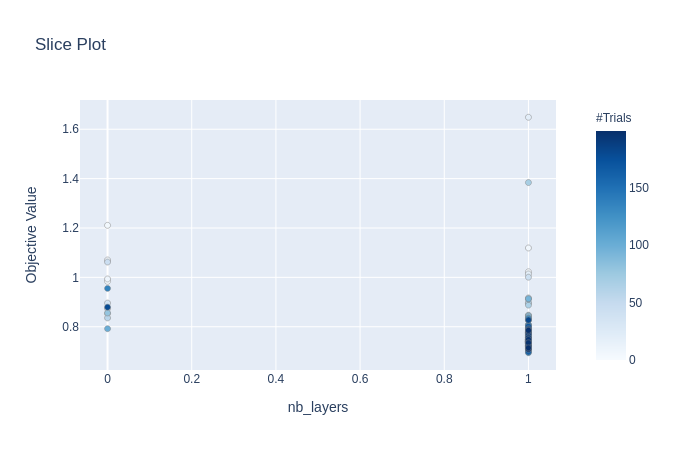}
    \caption{Number of hidden layers hyperparameter search. It motivated the use of perceptrons with one hidden layer.}
\end{figure}

\begin{figure}
    \centering
    \includegraphics[width=\textwidth]{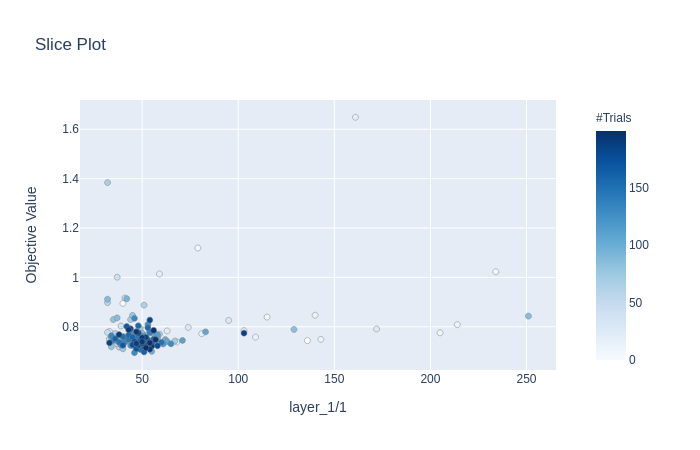}
    \caption{Size of the hidden layer in architectures with 1 hidden layer hyperparameter search. We chose $50$.}
\end{figure}

\begin{figure}
    \centering
    \includegraphics[width=\textwidth]{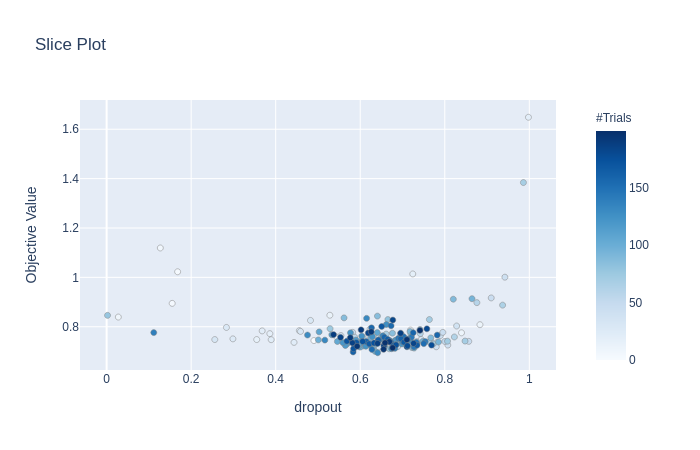}
    \caption{Dropout rate hyperparameter search. We chose $0.7$.}
\end{figure}

\begin{figure}
    \centering
    \includegraphics[width=\textwidth]{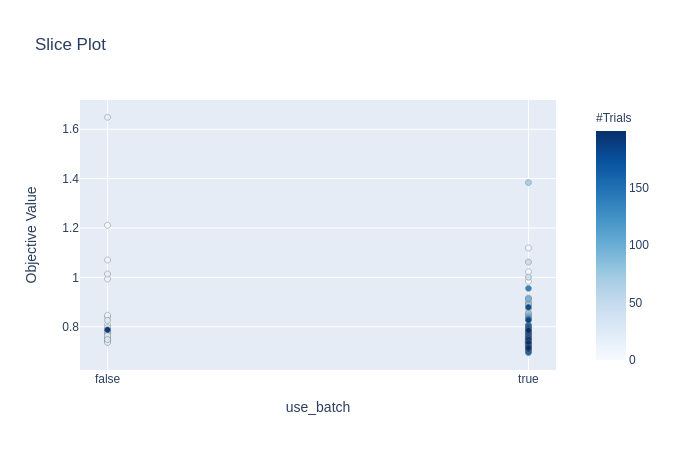}
    \caption{Batched optimization hyperparameter search. It shows the advantage of stochastic gradient descent.}
\end{figure}

\begin{figure}
    \centering
    \includegraphics[width=\textwidth]{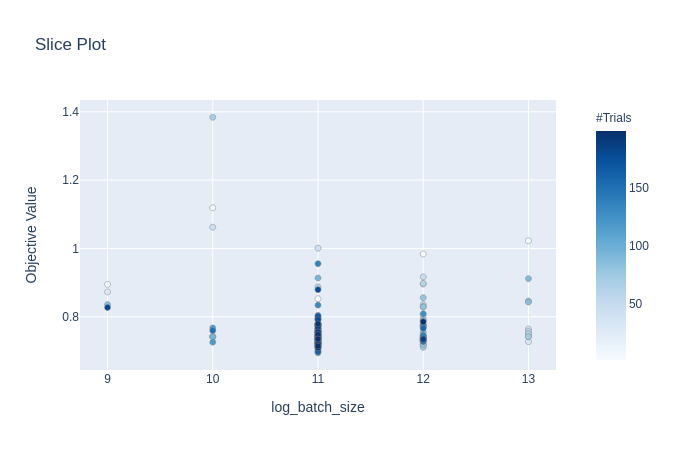}
    \caption{Batch size hyperparameter search. We choose $2^{12}$.}
\end{figure}

\begin{figure}
    \centering
    \includegraphics[width=\textwidth]{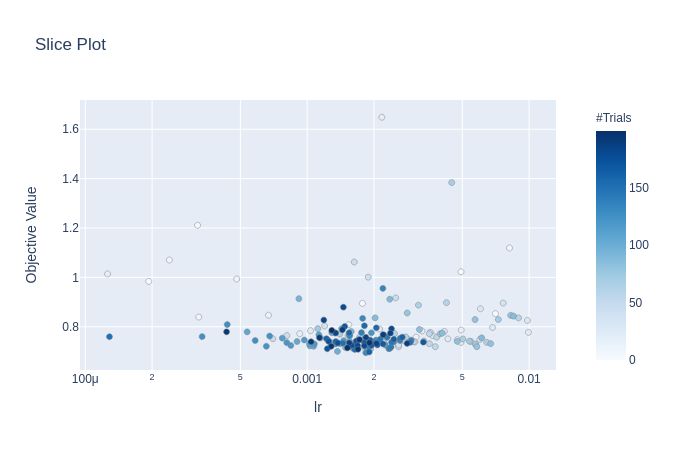}
    \caption{Learning rate hyperparameter search. We choose $2 \cdot 10^{-3}.$}
\end{figure}

\backmatter

\bibliographystyle{plain}
\bibliography{refs}


\end{document}